\documentclass[dvipsnames]{article} %
\usepackage{iclr2026_conference_unanon,times}
\usepackage{amsmath,amsthm,amssymb,bm}
\newtheorem{assumption}{Assumption}
\usepackage{booktabs}
\usepackage{tikz}
\usepackage{pgfplots}
\pgfplotsset{compat=1.17}
\usetikzlibrary{arrows.meta,positioning,fit,calc}
\usepackage{algorithm}
\usepackage{algpseudocode}
\usepackage{caption}
\usepackage{subcaption}
\usepackage{hyperref}
\usepackage[table]{xcolor}
\usepackage{multirow}
\usepackage{graphicx}
\usepackage{dsfont}
\usepackage{enumitem}
\usepackage{microtype}
\usepackage{makecell}
\usepackage{wrapfig}

\hypersetup{colorlinks=true, linkcolor=red!70!black, citecolor=blue!70!black, urlcolor=purple!70!black}

\newcommand{\cmark}{\textcolor{teal!80!black}{\checkmark}}
\newcommand{\xmark}{\textcolor{red!70!black}{$\times$}}

\setlist{nosep,leftmargin=1.5em}

\newcommand{\RKHS}{\mathcal{H}}
\newcommand{\Vars}{\mathcal{V}}
\newcommand{\dod}{\mathrm{do}}
\newcommand{\ind}[1]{\mathds{1}_{\{#1\}}}

\newtheorem{theorem}{Theorem}
\newtheorem{lemma}{Lemma}

\newtheorem{definition}{Definition}
\newtheorem{proposition}{Proposition}
\newtheorem*{remark}{Remark}

\title{Why LLMs Fail at Causal Discovery and How Interventional Agents Escape}

\author{%
  Amartya Roy\\
  SIRE, IIT Delhi and Robert Bosch GmbH, India\\
  \texttt{srz248670@iitd.ac.in}\\
 \And
  Sonali Parbhoo\\
  Imperial College London\\
  \texttt{s.parbhoo@imperial.ac.uk}
}

\begin{document}
\maketitle

% ══════════════════════════════════════════════════════════
\begin{abstract}
Causal discovery is a cornerstone of scientific reasoning, yet whether large language models can perform it reliably remains an open question. Recent benchmarks show that even fine-tuned models plateau on simple causal graphs and degrade as complexity grows, but why they fail has not been established. We prove the failure is fundamental: supervised fine-tuning, direct preference optimization, and in-context learning all produce predictors that cannot distinguish between causal graphs generating similar observational data, and any attempt to do so requires the model's internal representations to grow unboundedly, violating the very conditions under which these methods work. We formalize this as a kernel obstruction theorem, establishing that the limitation is intrinsic to the learning paradigm, \emph{not any particular model or dataset}. We propose Agentic Causal Bayesian Optimization (A-CBO), wherein a frozen language model serves as an interventional oracle answering targeted queries about intervention effects, while an external Bayesian loop concentrates beliefs over candidate graphs in logarithmically many rounds. Because the decision operates outside the space where the obstruction applies, A-CBO provably converges while the underlying model remains unchanged. On Corr2Cause, A-CBO matches fine-tuned baselines without any training. On Extended Corr2Cause, a new benchmark scaling to 24 variables with 18K test samples, A-CBO significantly outperforms both fine-tuning and preference optimization, with the advantage growing monotonically with graph complexity.
\end{abstract}

% ══════════════════════════════════════════════════════════
\section{Introduction}
\label{sec:intro}
% ══════════════════════════════════════════════════════════
Causal discovery, the problem of recovering causal graph structure from statistical data, is fundamental to science, medicine, and policy. E.g. A physician observing correlations between a drug, recovery, and patient age needs to know the underlying causal graph: does the drug cause recovery, or are both driven by age as a confounder? Answering this requires identifying which directed acyclic graph generated the observed statistics, a problem that is fundamentally underdetermined: multiple graphs can encode the same conditional independencies, forming a Markov equivalence class. \citet{scholkopf2021toward} establish that resolving such equivalences lies beyond the reach of statistical learning, which only captures associations under fixed conditions, and requires modelling interventions and counterfactuals. 

Large language models have achieved remarkable success across natural language tasks, raising a natural question: can they discover causal structure from statistical evidence? Recent benchmarks suggest they cannot. On Corr2Cause \citep{jin2023can}, which tests whether LLMs can identify causal graphs from correlational statements, even GPT-4 achieves only 29.1 macro-F1 in a zero-shot setting. Fine-tuned models reach high in-distribution accuracy but collapse under minimal perturbations like variable renaming \citep{jin2024corr2cause}, and performance degrades sharply as the number of variables grows. Several approaches have attempted to close this gap. Prompting strategies such as PC-SubQ \citep{sgouritsa2024prompting} decompose causal discovery into algorithmic sub-steps, structured thinking approaches \citep{sun2025structured} force explicit DAG construction before answering, and modular in-context methods \citep{kadziolka2025causal} embed the full PC algorithm within a single prompt. Interventional frameworks like LeGIT \citep{li2025can} pair LLMs with external causal discovery algorithms. While these methods yield empirical improvements on causal discovery benchmarks, these all retain the LLM as the entity that judges which causal graph is correct. More crucially, none explain \emph{why LLMs fail at this task}.

In this work, we prove that the failure is not a matter of insufficient data, model scale, or prompt engineering, but a fundamental geometric obstruction intrinsic to how these models learn. We show that supervised fine-tuning (SFT), direct preference optimization (DPO), and in-context learning (ICL) all produce kernel-type predictors whose ability to distinguish between causal graphs generating similar observational data is provably bounded. When two candidate graphs are observational near-misses, belonging to different Markov equivalence classes but producing nearly identical statistics, any attempt to separate them within these paradigms requires the model's internal representations to grow unboundedly, violating the very conditions under which these training methods are known to work. We formalize this as a \emph{kernel obstruction theorem} (Theorem~\ref{thm:kernel_obstruction}), establishing that the limitation is intrinsic to the learning paradigm, not to any particular model or dataset.

We then ask: \emph{what is the minimal construction that escapes this obstruction?} The answer, dictated by the theorem itself, is to move the discrete graph-selection decision outside the kernel predictor entirely. We propose Agentic Causal Bayesian Optimization (A-CBO), an approach wherein a frozen language model serves as an interventional oracle, answering simple binary questions such as ``does $V_j$ change under $\mathrm{do}(V_i = v)$?'', while an external Bayesian loop uses these answers to concentrate beliefs over candidate causal graphs. The LLM never determines which graph is correct; it only answers factual questions about the effects of interventions, questions whose answers happen to differ across competing hypotheses. The Bayesian update operates in the probability simplex $\Delta_{n-1}$, outside the RKHS where the obstruction applies, so A-CBO provably converges to the correct graph in logarithmically many rounds while the LLM itself remains unchanged.

\textbf{Contributions.}
\textbf{(1)}~We prove a kernel obstruction theorem showing that SFT, DPO, and ICL cannot separate near-miss causal hypotheses within the lazy regime. \textbf{(2)}~We propose A-CBO, an agentic framework that sidesteps this obstruction via interventional queries while preserving the lazy regime. \textbf{(3)}~We evaluate on Corr2Cause and introduce Extended Corr2Cause ($d = 7$--$24$), a new causal
benchmark scaling to 24 variables, demonstrating that A-CBO matches finetuned baselines 
without retraining and outperforms both SFT and DPO by an average of 24\%, with the 
advantage growing monotonically with graph complexity. Table~\ref{tab:method_comparison} summarises where A-CBO differs from prior work across four key dimensions. 

% \begin{table}[h]
% \centering
% \vspace{-2pt}
% \caption{\textbf{A-CBO is the only approach that simultaneously satisfies all four properties.} Comparison of causal discovery methods on dimensions that govern reliability at scale.}
% \label{tab:method_comparison}
% \vspace{2pt}
% \resizebox{0.8\textwidth}{!}{%
% \scriptsize
% \setlength{\tabcolsep}{5pt}
% \renewcommand{\arraystretch}{1.12}
% \begin{tabular}{@{} l cccc @{}}
% \toprule
% \textbf{Method} & \makecell{\textbf{Training-}\\\textbf{free}} & \makecell{\textbf{Near-miss}\\\textbf{separation}} & \makecell{\textbf{Scales to}\\\textbf{$d\!=\!24$}} & \makecell{\textbf{Convergence}\\\textbf{guarantee}} \\
% \midrule
% Zero-shot (GPT-4)   & \cmark & \xmark & \xmark & \xmark \\
% ICL / Prompting     & \cmark & \xmark & \xmark & \xmark \\
% SFT                 & \xmark & \xmark & \xmark & \xmark \\
% DPO                 & \xmark & \xmark & \xmark & \xmark \\
% \midrule
% \rowcolor{blue!8}
% \textbf{A-CBO (ours)} & \cmark & \cmark & \cmark & \cmark \\
% \bottomrule
% \end{tabular}%
% }
% \vspace{-4pt}
% \end{table}

\begin{table}[h]
\centering
\vspace{-2pt}
\caption{\textbf{A-CBO is the only approach that simultaneously satisfies all four properties.} Comparison of causal discovery(for detailed discussion please refer to Appendix~\ref{app:why-causal-discovery}) methods on dimensions that govern reliability at scale.}
\label{tab:method_comparison}
\vspace{2pt}
\scriptsize
\setlength{\tabcolsep}{5pt}
\renewcommand{\arraystretch}{1.12}
\begin{tabular}{@{} l cccc @{}}
\toprule
\textbf{Method} & \makecell{\textbf{Training-}\\\textbf{free}} & \makecell{\textbf{Near-miss}\\\textbf{separation}} & \makecell{\textbf{Scales to}\\\textbf{$d\!=\!24$}} & \makecell{\textbf{Convergence}\\\textbf{guarantee}} \\
\midrule
Zero-shot (GPT-4)   & \cmark & \xmark & \xmark & \xmark \\
ICL / Prompting     & \cmark & \xmark & \xmark & \xmark \\
SFT                 & \xmark & \xmark & \xmark & \xmark \\
DPO                 & \xmark & \xmark & \xmark & \xmark \\
\midrule
\rowcolor{blue!8}
\textbf{A-CBO (ours)} & \cmark & \cmark & \cmark & \cmark \\
\bottomrule
\end{tabular}
\vspace{-4pt}
\end{table}

% ── Figure 1: Overview ────────────────────────────────────
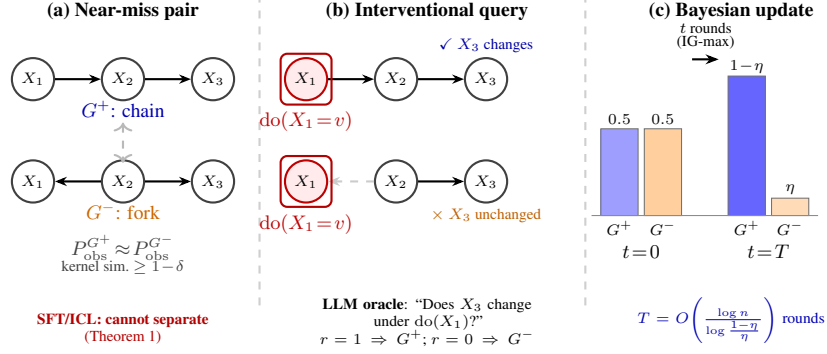
\begin{figure}[t]
\centering
\begin{tikzpicture}[
    scale=0.82,
    transform shape,
    var/.style={circle,draw=black!75,thick,minimum size=0.68cm,font=\scriptsize\bfseries,inner sep=1pt},
    intv/.style={circle,draw=red!75!black,fill=red!8,thick,minimum size=0.68cm,font=\scriptsize\bfseries,inner sep=1pt},
    arr/.style={-{Stealth[length=4.5pt,width=3.5pt]},thick}
]
\useasboundingbox (-0.6,-1.15) rectangle (13.1,4.85);

%% PANEL (a)
\node[font=\small\bfseries] at (1.55,4.50) {(a)~Near-miss pair};
\node[var] (A1) at (0.1,3.4) {$X_1$};
\node[var] (A2) at (1.55,3.4) {$X_2$};
\node[var] (A3) at (3.0,3.4) {$X_3$};
\draw[arr] (A1) -- (A2); \draw[arr] (A2) -- (A3);
\node[font=\small,text=blue!70!black] at (1.55,2.92) {$G^+$: chain};

\node[var] (B1) at (0.1,1.75) {$X_1$};
\node[var] (B2) at (1.55,1.75) {$X_2$};
\node[var] (B3) at (3.0,1.75) {$X_3$};
\draw[arr] (B2) -- (B1); \draw[arr] (B2) -- (B3);
\node[font=\small,text=orange!80!black] at (1.55,1.27) {$G^-$: fork};

\draw[<->,gray!55,dashed,thick] (1.55,2.70) -- (1.55,2.05);
\node[font=\small,text=gray!50!black,align=center] at (1.55,0.58)
    {$P_{\mathrm{obs}}^{G^+}\!\approx\!P_{\mathrm{obs}}^{G^-}$\\[-3pt]\scriptsize kernel sim.~$\ge 1\!-\!\delta$};
\node[font=\scriptsize,text=red!70!black,align=center] at (1.55,-0.62)
    {\textbf{SFT/ICL: cannot separate}\\(Theorem~\ref{thm:kernel_obstruction})};
\draw[dashed,gray!38,line width=0.8pt] (3.75,-0.25) -- (3.75,4.70);

%% PANEL (b)
\node[font=\small\bfseries] at (6.45,4.50) {(b)~Interventional query};
\node[intv] (C1) at (4.5,3.4) {$X_1$};
\node[var]  (C2) at (5.95,3.4) {$X_2$};
\node[var]  (C3) at (7.4,3.4) {$X_3$};
\draw[arr] (C1) -- (C2); \draw[arr] (C2) -- (C3);
\draw[red!70!black,thick,rounded corners=2.5pt] (4.09,2.99) rectangle (4.91,3.81);
\node[font=\small,text=red!75!black] at (4.5,2.72) {$\dod(X_1\!=\!v)$};
% FIXED: moved label above the node row
\node[font=\scriptsize,text=blue!65!black] at (7.4,3.95) {$\checkmark$ $X_3$ changes};

\node[intv] (D1) at (4.5,1.75) {$X_1$};
\node[var]  (D2) at (5.95,1.75) {$X_2$};
\node[var]  (D3) at (7.4,1.75) {$X_3$};
\draw[arr,gray!40,dashed] (D2) -- (D1); \draw[arr] (D2) -- (D3);
\draw[red!70!black,thick,rounded corners=2.5pt] (4.09,1.34) rectangle (4.91,2.16);
\node[font=\small,text=red!75!black] at (4.5,1.07) {$\dod(X_1\!=\!v)$};
% FIXED: moved label below the node row
\node[font=\scriptsize,text=orange!75!black] at (7.4,1.20) {$\times$ $X_3$ unchanged};

\node[font=\scriptsize,align=center,text width=4.2cm] at (6.45,-0.52)
    {\textbf{LLM oracle}: ``Does $X_3$ change\\under $\dod(X_1)$?''\\[-1pt]$r\!=\!1\Rightarrow G^+$;\;$r\!=\!0\Rightarrow G^-$};
\draw[dashed,gray!38,line width=0.8pt] (9.0,-0.25) -- (9.0,4.70);

%% PANEL (c)
%% PANEL (c)
\node[font=\small\bfseries] at (11.35,4.50) {(c)~Bayesian update};
\pgfmathsetmacro{\yb}{1.2}
\fill[blue!38!white]   (9.25,\yb) rectangle (9.85,2.60);
\fill[orange!38!white] (9.95,\yb) rectangle (10.55,2.60);
\draw[black!55] (9.25,\yb) rectangle (9.85,2.60);
\draw[black!55] (9.95,\yb) rectangle (10.55,2.60);
\node[font=\scriptsize] at (9.55,1.0)  {$G^+$};
\node[font=\scriptsize] at (10.25,1.0) {$G^-$};
\node[font=\small]      at (9.90,0.62) {$t\!=\!0$};
\node[font=\scriptsize] at (9.55,2.76)  {$0.5$};
\node[font=\scriptsize] at (10.25,2.76) {$0.5$};

% FIXED: arrow above both bar groups, label above arrow
\node[font=\scriptsize,align=center] at (10.94,4.10) {$t$ rounds\\[-2pt](IG-max)};
\draw[-{Stealth[length=5pt]},thick] (10.72,3.72) -- (11.15,3.72);

\fill[blue!60!white]   (11.3,\yb) rectangle (11.9,3.45);
\fill[orange!28!white] (12.0,\yb) rectangle (12.6,1.48);
\draw[black!55] (11.3,\yb) rectangle (11.9,3.45);
\draw[black!55] (12.0,\yb) rectangle (12.6,1.48);
\node[font=\scriptsize] at (11.6,1.0)  {$G^+$};
\node[font=\scriptsize] at (12.3,1.0)  {$G^-$};
\node[font=\small]      at (11.95,0.62) {$t\!=\!T$};
\node[font=\scriptsize] at (11.6,3.61)  {$1\!-\!\eta$};
\node[font=\scriptsize] at (12.3,1.62)  {$\eta$};

\node[font=\scriptsize,text=blue!70!black,align=center,text width=3.2cm] at (11.35,-0.48)
    {$T = O\!\left(\frac{\log n}{\log\frac{1-\eta}{\eta}}\right)$ rounds};
\draw[black!45,thin] (9.1,\yb) -- (12.75,\yb);
\end{tikzpicture}
\vspace{-4pt}
\caption{\textbf{Overview of A-CBO.}
\textbf{(a)}~Two near-miss hypotheses ($G^+$: chain, $G^-$: fork) are observationally equivalent; kernel similarity $\ge 1\!-\!\delta$, so bounded-norm SFT/ICL cannot separate them (Thm.~\ref{thm:kernel_obstruction}).
\textbf{(b)}~A single intervention $\dod(X_1\!=\!v)$ discriminates: under $G^+$ the perturbation propagates to $X_3$; under $G^-$ the severed edge leaves $X_3$ unaffected.
\textbf{(c)}~A-CBO performs Bayesian updates in $\Delta^{n-1}$ (outside $\RKHS$), concentrating belief on the correct hypothesis in $O(\log n)$ rounds.}
\label{fig:overview}
\end{figure}

% ══════════════════════════════════════════════════════════
% ══════════════════════════════════════════════════════════
\section{Related Work}
\label{sec:rel_work}
% ══════════════════════════════════════════════════════════
% ══════════════════════════════════════════════════════════
We provide a succinct description of related work here. A more detailed related work is provided in Appendix~\ref{app:related_work}.

\textbf{Causal reasoning in LLMs.}
A growing body of work documents that LLMs struggle with genuine causal inference. \citet{zecevic2023causal} argue that LLMs are ``causal parrots,'' reciting correlations over causal facts from training data rather than performing inference. \citet{zhang2023understanding} find that LLMs succeed at retrieving memorised causal relations but fail at inferring novel ones from statistical evidence. \citet{wu2024causality} survey five integration stages, from data augmentation to architecture modification, and find persistent shortfalls at every stage. Most directly relevant to our work, \citet{wu2025llm} argue on empirical and conceptual grounds that LLMs should be restricted to non-decisional support in causal discovery, never determining the existence or directionality of causal relationships.  Our kernel obstruction theorem (Theorem~\ref{thm:kernel_obstruction}) provides the formal proof underlying this empirical observation, and A-CBO instantiates precisely the non-decisional oracle role they advocate.

\textbf{Benchmarks for causal reasoning.} \citet{jin2024corr2cause} introduced \textsc{Corr2Cause}, showing that seventeen LLMs perform near random on pure causal inference from correlational statements, while fine-tuned models achieve high in-distribution accuracy but collapse under perturbations such as variable renaming. \citet{chi2024unveiling} use fresh news data to distinguish shallow associative reasoning from genuine causal inference, finding that strong benchmark performance largely reflects memorisation. \citet{sheth2025causalgraph2llm} evaluate how graph encoding strategies affect LLM causal performance, 
finding high sensitivity to format choices. \citet{yamin2024failure} show that LLMs rely on superficial heuristics such as event ordering rather than structural reasoning when processing causal narratives. These benchmarks document the empirical phenomenon our theory explains: degradation that worsens with structural complexity and near-miss configurations. We extend this foundation with a new benchmark \textsc{Extended Corr2Cause} scaling from 7 to 24 variables with 18K samples, and show that the degradation pattern matches the quantitative predictions of our kernel obstruction theorem.

\textbf{LLM-assisted causal discovery with interventions.} LeGIT \citep{li2025can} uses LLM-proposed intervention targets to warm-start external causal discovery algorithms when observational data is limited. \citet{le2024multi} combine a multi-agent debate module with statistical validation, and \citet{abdulaal2024causal} alternate between LLM-proposed graph 
structures and deep SCM fitting. All three retain the LLM as a reasoning agent that makes or refines causal judgments. A-CBO differs in both motivation and mechanism: our kernel obstruction thereom dictates that the discrete causal decision must reside \emph{outside} the LLM, which serves only as a fixed binary oracle while an external Bayesian loop in $\Delta^{n-1}$ performs all hypothesis discrimination. LeGIT is closest in spirit but uses the LLM to \emph{select} interventions based on semantic variable knowledge; A-CBO uses information-theoretic scoring to 
select and the LLM only to evaluate.
% ══════════════════════════════════════════════════════════
% ══════════════════════════════════════════════════════════
\section{Mathematical Setup}
\label{sec:setup}
% ══════════════════════════════════════════════════════════

\textbf{Structural Causal Models.}
An SCM over $\mathcal{V} = \{V_1, \ldots, V_d\}$ is a DAG $G = (\mathcal{V}, E)$ with structural equations $V_j = f_j(\mathrm{pa}_G(V_j), \varepsilon_j)$. An \emph{intervention} $\mathrm{do}(V_i = v)$ replaces $V_i$'s equation with the constant $v$ (graph mutilation), inducing an interventional distribution $P^{\mathrm{do}(V_i = v)}$. Two DAGs are \emph{Markov equivalent} if they encode the same conditional independencies; resolving equivalences requires reasoning about interventions.

\textbf{Kernel-Type Predictors \& Lazy Training.}
A PSD kernel $K : \mathcal{Z}^2 \to \mathbb{R}$ induces an RKHS $\mathcal{H}$ with feature map $\phi$. A \emph{kernel-type predictor} is $s(z) = \langle w, \phi(z) \rangle_{\mathcal{H}}$ with $\|w\|_{\mathcal{H}} \leq B$. In the \emph{lazy/NTK regime}~\citep{jacot2018neural}, the NTK $\Theta((x_o,y_o),(x_u,y_u)) = \langle \nabla_\theta z_\theta(x_o,y_o), \nabla_\theta z_\theta(x_u,y_u)\rangle$ stays approximately constant, so predictions evolve as a kernel smoother. SFT produces such a predictor via the NTK update rule (Appendix~\ref{app:sft_kernel}); ICL approximates kernel regression~\citep{han2023explaining,sun2025role}; DPO inherits the same structure under its KL constraint. All three thus satisfy $B = O(1)$.

\textbf{Problem Formulation.}
Given $d$ variables generated by unknown DAG $G$, the model receives premises $\mathcal{P} = \{p_1,\ldots,p_m\}$ (pairwise correlation/independence statements) and a causal hypothesis $h$. It must output $\ell = 1$ if $h$ is entailed by $\mathcal{P}$ in every consistent DAG, and $\ell = 0$ otherwise. 
Given input $x = (\mathcal{P}, h)$, the LLM's log-score over a reasoning chain $y = (y_1,\ldots,y_L)$ is:
$
s_\theta(x, y) \triangleq \log \pi_\theta(y \mid x) = \sum_{l=1}^{L} \log \pi_\theta(y_l \mid x, y_{<l}).
\label{eq:log_score}
$
Reliable discrimination requires a score margin:
\begin{equation}
s_\theta(x,y^+) - s_\theta(x,y^-) \ge \gamma > 0.
\label{eq:margin_req}
\end{equation}

\textbf{Key Definitions.} Let $\mathcal{Z} = \mathcal{X}\times\mathcal{Y}$ and $\chi=(x,y)$ with kernel $K(\chi,\chi)\le\kappa^2$.

\begin{definition}[$\delta$-similar pair]
\label{def:delta_similar}
Pairs $\chi^+ = (x, y^+)$ and $\chi^- = (x, y^-)$ are $\delta$-similar if
$K(\chi^+, \chi^-)/\sqrt{K(\chi^+, \chi^+) \cdot K(\chi^-, \chi^-)} \geq 1 - \delta$.
\end{definition}

\begin{definition}[Structural Discrimination Set]
\label{def:discrimination_set}
For DAGs $G^+, G^-$ over $\mathcal{V}$: $\mathcal{D}(G^+, G^-) \triangleq \{(V_i, V_j) \in \mathcal{V}^2 : \hat{r}_{G^+}(V_i, V_j) \neq \hat{r}_{G^-}(V_i, V_j)\}$, where $\hat{r}_G(V_i,V_j)\in\{0,1\}$ is the graph-mutilation prediction of whether $V_j$ is affected by $\mathrm{do}(V_i=v)$.
\end{definition}

\begin{assumption}[Oracle Reliability]
\label{ass:oracle}
The LLM oracle answers interventional queries correctly with probability $1-\eta > 1/2$. The true DAG $G^\star$ is interventionally distinguishable from every candidate: $\mathcal{D}(G^\star, G_k)\neq\emptyset$ for all $k\neq\star$.
\end{assumption}

% ══════════════════════════════════════════════════════════
\section{Kernel Obstruction and Its Escape}
\label{sec:theory}
\label{sec:methods}
% ══════════════════════════════════════════════════════════

\textbf{The Near-Miss Problem.}
A chain $V_1\!\to\!V_2\!\to\!V_3$ and a fork $V_1\!\leftarrow\!V_2\!\to\!V_3$ produce identical marginal correlations ($V_1\!\perp\!\!\!\perp\!V_3\mid V_2$ in both), so their textual premise $\mathcal{P}$ is \emph{identical}, yet they require opposite labels. As $d$ grows, near-miss DAGs share $1-O(1/d^2)$ of their tokens; at $d=24$, $>$99\% overlap. The distinguishing information is present the DAGs are \emph{not} Markov equivalent---but it occupies a vanishing fraction of the input, making discrimination geometrically impossible for textual-similarity-based predictors.

\subsection{Helping Lemmas}

\begin{lemma}[Near-Miss Kernel Similarity]
\label{lem:token_overlap}
For near-miss pairs sharing a token prefix of length $\ell$ out of total $L=O(d^2)$ tokens, the NTK similarity satisfies $\delta \leq C(L-\ell)/L$ for a constant $C>0$. Since $L-\ell = O(1)$, we get $\delta = O(1/d^2) \to 0$ as $d\to\infty$.
\end{lemma}
\noindent\emph{Proof in Appendix~\ref{app:proofs}.}

\vspace{4pt}
Interventions break the near-miss symmetry: under a chain $V_i\to V_k\to V_j$, intervening on $V_i$ propagates to $V_j$; under a fork $V_i\leftarrow V_k\to V_j$, the path is severed and $V_j$ is unaffected. Binary oracle responses are \emph{not} near-misses, they remain kernel-separated regardless of $\delta$.

\begin{lemma}[\textbf{Interventional Kernel Separation}]
\label{lem:interventional_separation}
Let $(V_i, V_j) \in \mathcal{D}(G^+, G^-)$ (Definition~\ref{def:discrimination_set}). The interventional query ``\texttt{Does $V_j$ change under $\mathrm{do}(V_i = v)$?}'' produces responses $r^+ = 1$ and $r^- = 0$ whose kernel representations satisfy:
\begin{equation}
    \frac{K(\chi_{\mathrm{yes}}, \chi_{\mathrm{no}})}
         {\sqrt{K(\chi_{\mathrm{yes}}, \chi_{\mathrm{yes}})
                \cdot K(\chi_{\mathrm{no}}, \chi_{\mathrm{no}})}}
    \leq 1 - \rho,
\end{equation}
for a constant $\rho \in (0,1]$ depending on $\mathcal{D}(G^+, G^-)$ but \emph{not} on $\delta$. Hence $\rho \not\to 0$ as $\delta \to 0$: the LLM reliably answers interventional queries even when it cannot answer the global causal question.
\end{lemma}
\begin{proof}[Proof sketch]
The key decoupling: $\rho$ depends on structural disagreement via $\mathcal{D}(G^+,G^-)$ while $\delta$ depends on observational similarity. Choosing $(V_i,V_j)\in\mathcal{D}$ to maximise disagreement ensures $\rho\not\to 0$ as $\delta\to 0$. Full proof in Appendix~\ref{app:proofs}.
\end{proof}

\subsection{Main Results}

With both helping lemmas established, the two main results follow directly. Theorem~\ref{thm:kernel_obstruction} proves the impossibility for all kernel-based methods; Theorem~\ref{thm:convergence} proves that A-CBO escapes it. Full proofs are in Appendix~\ref{app:proofs}.

\begin{theorem}[\textbf{Kernel Obstruction for Causal Discrimination}]
\label{thm:kernel_obstruction}
Let $\chi^+ = (x, y^+)$ and $\chi^- = (x, y^-)$ be a $\delta$-similar pair. For any scoring rule of the form $s(\chi) = \langle w, \phi(\chi) \rangle_{\mathcal{H}}$ with $\|w\|_{\mathcal{H}} \leq B$:
\begin{equation}
    s(x, y^+) - s(x, y^-) \leq B\kappa\sqrt{2\delta}.
\end{equation}
Consequently, enforcing the discrimination margin $\gamma$ from Eq.~\eqref{eq:margin_req} requires
$B \geq \gamma / (\kappa\sqrt{2\delta})$, which diverges as $\delta \to 0$.
\end{theorem}

\begin{proof}[Proof sketch]
By Cauchy--Schwarz, $|s(\chi^+)-s(\chi^-)|\le \|w\|_\RKHS\cdot\|\phi(\chi^+)-\phi(\chi^-)\|_\RKHS$. The $\delta$-similarity bound yields $\|\phi(\chi^+)-\phi(\chi^-)\|^2 \le 2\kappa^2\delta$, giving the upper bound. Enforcing margin $\gamma$ forces $B\ge \gamma/(\kappa\sqrt{2\delta})\to\infty$. Full proof in Appendix~\ref{app:proofs}.
\end{proof}

Theorem~\ref{thm:kernel_obstruction} applies uniformly to SFT (via the NTK update rule, Appendix~\ref{app:sft_kernel}), DPO (whose implicit reward inherits the same RKHS structure under the KL constraint), and ICL (which approximates kernel regression~\citep{han2023explaining,sun2025role}). As $d$ grows, $\delta=O(1/d^2)\to 0$ (Lemma~\ref{lem:token_overlap}), so the achievable margin $B\kappa\sqrt{2\delta}=O(1/d)\to 0$: \emph{any kernel-based method degrades with graph complexity, not from insufficient data or capacity, but from geometric impossibility.}

% ══════════════════════════════════════════════════════════
\section{A-CBO: The Constructive Escape}
\label{sec:acbo}

Theorem~\ref{thm:kernel_obstruction} is agnostic to model choice, dataset, or training procedure: any method computing causal judgments \emph{within} a bounded-norm kernel predictor cannot separate near-miss hypotheses. A-CBO poses local binary interventional queries - \texttt{``Does $V_j$ change under $\mathrm{do}(V_i{=}v)$?''} whose yes/no answers are kernel-separated (Lemma~\ref{lem:interventional_separation}), delegating all hypothesis discrimination to an external Bayesian loop in $\Delta^{n-1}$, outside the RKHS entirely (Figure~\ref{fig:overview}).

% \begin{algorithm}[htbp!]
% \caption{Agentic Causal Bayesian Optimization (A-CBO)}
% \label{alg:acbo}
% \begin{algorithmic}[1]
% \Require Premise $\mathcal{P}$, frozen LLM $\mathcal{L}_{\theta_0}$, variables $\Vars$, target $V_t$, budget $T$, oracle noise $\eta$, votes $M$
% \Ensure MAP hypothesis $G^\star$
% \State \textbf{Phase 1 (Hypothesis Generation):} $\{G_k\}_{k=1}^n \gets \mathcal{L}_{\theta_0}(\text{prompt}_\text{gen}(\mathcal{P}))$;\; $\pi_k^{(0)} \gets 1/n$
% \State \textbf{Phase 2 (Discrimination Loop):}
% \For{$t = 1,\ldots,T$}
%     \If{$\max_k\pi_k^{(t-1)} > 1-\delta_c$} \textbf{break} \EndIf
%     \State \textit{// A: Score interventions (no LLM call)} \;\; Compute $\hat{r}_k(V_i,V_j)$ via graph mutilation for alive $G_k$
%     \State $(V_i^\star,V_j^\star) \gets \arg\max_{V_i,V_j} \mathrm{IG}(V_i,V_j)$,\; where $\mathrm{IG} = H(\bm{\pi}^{(t-1)}) - \sum_r P(r)\,H(\bm{\pi}^{(t-1)}\mid r)$
%     \State \textit{// B: LLM oracle query}\;\; $r^{\text{obs}} \gets \text{MajVote}\bigl(\mathcal{L}_{\theta_0}(\text{prompt}_\text{int}(\mathcal{P},\dod(V_i^\star),V_j^\star,V_t))\bigr)$
%     \State \textit{// C: Bayesian update in $\Delta^{n-1}$}\;\; $\pi_k^{(t)} \propto \pi_k^{(t-1)}\cdot\bigl[(1-\eta)\ind{\hat{r}_k = r^\text{obs}} + \eta\ind{\hat{r}_k \ne r^\text{obs}}\bigr]$
% \EndFor
% \State \Return $G^\star = \arg\max_k \pi_k^{(T)}$
% \end{algorithmic}
% \end{algorithm}

\begin{algorithm}[htbp!]
\caption{Agentic Causal Bayesian Optimization (A-CBO)}
\label{alg:acbo}
\begin{algorithmic}[1]
\Require Premise $\mathcal{P}$, frozen LLM $\mathcal{L}_{\theta_0}$, variables $\Vars$, target $V_t$, budget $T$, oracle noise $\eta$, votes $M$
\Ensure MAP hypothesis $G^\star$
\State \textbf{Phase 1: Hypothesis Generation} (lazy-regime forward pass)
\State $\{G_1,\ldots,G_n\} \gets \mathcal{L}_{\theta_0}(\text{prompt}_\text{gen}(\mathcal{P}))$;\quad $\pi_k^{(0)} \gets 1/n$
\State \textbf{Phase 2: Agentic Discrimination Loop}
\For{$t = 1,\ldots,T$}
    \If{converged ($\max_k\pi_k^{(t-1)} > 1-\delta_c$)} \textbf{break} \EndIf
    \State \textit{// Step A: Score interventions (deterministic, no LLM call)}
    \For{each $(V_i,V_j)$}
        \State Compute $\hat{r}_k(V_i,V_j)$ for alive hypotheses via graph mutilation
        \State Compute $\mathrm{IG}(V_i,V_j) = H(\bm{\pi}^{(t-1)}) - \sum_r P(r)\,H(\bm{\pi}^{(t-1)}\mid r)$
    \EndFor
    \State $(V_i^\star,V_j^\star) \gets \arg\max_{V_i,V_j} \mathrm{IG}(V_i,V_j)$
    \State \textit{// Step B: LLM oracle query (lazy-regime forward pass)}
    \State $r^{\text{obs}} \gets \text{MajVote}\bigl(\mathcal{L}_{\theta_0}(\text{prompt}_\text{int}(\mathcal{P},\dod(V_i^\star),V_j^\star,V_t))\bigr)$
    \State \textit{// Step C: External Bayesian update (in $\Delta^{n-1}$, not $\RKHS$)}
    \For{each $G_k$}
        $\pi_k^{(t)} \propto \pi_k^{(t-1)}\cdot\bigl[(1-\eta)\ind{\hat{r}_k\!=\!r^\text{obs}} + \eta\ind{\hat{r}_k\!\ne\!r^\text{obs}}\bigr]$
    \EndFor
\EndFor
\State \Return $G^\star = \arg\max_k \pi_k^{(T)}$
\end{algorithmic}
\end{algorithm}

\begin{theorem}[\textbf{Convergence of A-CBO}]
\label{thm:convergence}
Suppose the true DAG $G^\star$ is interventionally distinguishable from every candidate $G_k$ 
($\mathcal{D}(G^\star, G_k) \neq \emptyset\; \forall k \neq \star$), and the LLM oracle answers  correctly with probability $1 - \eta > 1/2$. Then Algorithm~\ref{alg:acbo} identifies $G^\star$  in at most $T^\star = \left\lceil \frac{\log n}{\log \frac{1-\eta}{\eta}} \right\rceil$ rounds with probability $\geq 1 - n\eta^{T^\star}$, remaining in the lazy/NTK regime throughout. Crucially, $T^\star$ depends on the number of hypotheses $n$ and oracle quality $\eta$, but is \emph{independent of the near-miss parameter $\delta$}.
\end{theorem}

\begin{proof}[Proof sketch]
At each round, the correct hypothesis $G^\star$ accumulates multiplicative weight $(1-\eta)/\eta > 1$ relative to every wrong hypothesis. After $T^\star$ rounds, this geometric growth ensures $\pi_{G^\star}^{(T^\star)} \geq 1 - n\eta^{T^\star}$ by a union bound. The loop operates in $\Delta^{n-1}$ via Bayes' rule, never invoking a kernel comparison; hence $T^\star$ is independent of $\delta$. Full proof in Appendix~\ref{app:proofs}.
\end{proof}

The independence from $\delta$ is the central point. Where SFT, DPO, and ICL all degrade as $\delta \to 0$ with growing graph complexity, A-CBO's convergence rate is unaffected. The near-miss geometry that defeats kernel predictors is simply not visible to the Bayesian loop, because the loop never operates in the space where the obstruction applies. This is not an empirical observation but a structural guarantee: A-CBO succeeds on near-miss instances where Theorem~\ref{thm:kernel_obstruction} proves the alternatives cannot.

% ══════════════════════════════════════════════════════════
\section{Experiments}
\label{sec:experiments}
\subsection{Setup}
\label{sec:exp_setup}
% ──────────────────────────────────────────────────────────
We evaluate A-CBO on two causal discrimination benchmarks against zero-shot,
fine-tuned, and preference-optimised baselines. All A-CBO configurations use
\emph{frozen} LLMs with zero gradient updates throughout. The full experimental setup benchmarks, baselines, model tiers, hyperparameters, and evaluation metrics are summarised in Table~\ref{tab:setup}.

\textbf{Datasets} We evaluate on two benchmarks. The first is \textsc{Corr2Cause} \citep{jin2024corr2cause}, a dataset of 7,524 test samples spanning causal graph depths $d \in \{2, \ldots, 6\}$, with instances drawn from six causal relation templates: \texttt{Parent}, \texttt{Child}, \texttt{Ancestor}, \texttt{Descendant}, \texttt{Collider}, and \texttt{Confounder}. Performance is measured via macro-averaged F1 across all six classes. The second is our proposed \textsc{Extended C2C} benchmark, comprising 18,000 samples at depths $d \in \{7, \ldots, 24\}$ (1,000 samples per depth level). All instances carry all-negative labels, and the task reduces to binary rejection accuracy. Crucially, the near-miss gap between the correct and closest incorrect causal relation shrinks as $O(1/d^2)$, making this benchmark a stress test of fine-grained causal discrimination at scale.

\textbf{Baselines.} We compare against four baselines spanning prompting and supervised fine-tuning paradigms. As a prompting baseline, we include \textbf{zero-shot GPT-4}, applied via direct prompting without any A-CBO loop, which achieves a macro-F1 of 29.1, illustrating the difficulty of the task for large proprietary models without structured reasoning support. We also report \textbf{LLaMA-7B (FT)}, (from the paper~\citet{jin2023can}) which reaches a macro-F1 of 92.0 on \textsc{Corr2Cause}. On the \textsc{Extended C2C} benchmark, we evaluate two additional supervised baselines: \textbf{RoBERTa-Large SFT} (355M parameters), trained on 1.3M \textsc{Extended C2C} samples for 3 epochs using AdamW with a learning rate of $2 \times 10^{-5}$ and batch size 32; and \textbf{RoBERTa-Large DPO}, trained with preference pairs constructed per instance using $\beta = 0.1$, optimized on $4\times$A100 GPUs.

\textbf{A-CBO Models.}
We instantiate the A-CBO framework across three capability tiers to assess how backbone model strength interacts with the optimization loop. The \textit{high} tier comprises GLM-5.1$^\star$ and Qwen3-30B$^\star$, both operated with thinking mode enabled.($\star$ Thinking mode(Across Setting) The \textit{mid} tier consists of Qwen3.5-122B and Llama-3.3-70B. The \textit{low} tier uses Gemma-3-12B-IT and LLaMA-7B. This stratification allows us to disentangle the contribution of the A-CBO procedure itself from the raw capability of the underlying language model.

\textbf{Hyperparameters.}
The A-CBO loop runs for a maximum of $T = 20$ iterations with a random exploration fraction of $\varepsilon = 0.1$ and oracle noise level $\eta = 0.1$. Convergence is determined by an entropy-based stopping criterion with threshold $\delta_c = 0.01$: optimization halts when the entropy of the candidate distribution falls below this value. Final relation predictions are aggregated via majority voting over $M = 3$ votes drawn from $N = 8$ candidates per iteration.

\textbf{Evaluation Metrics.}
On \textsc{Corr2Cause}, we report macro-averaged F1 across the six causal relation classes, consistent with the original benchmark protocol. On \textsc{Extended C2C}, we report binary rejection accuracy, reflecting whether the model correctly identifies that no valid causal relation holds among the presented variables.

% ──────────────────────────────────────────────────────────
\subsection{Results on Corr2Cause and Extended Corr2Cause}
\label{sec:c2c_results}
% ──────────────────────────────────────────────────────────

Table~\ref{tab:main_results} reports per-class F1 and overall metrics on the
original benchmark; Table~\ref{tab:ext_results_main} reports the same on the
extended benchmark, averaged across $d\!=\!7$--$24$.

\textbf{A-CBO matches fine-tuned baselines without retraining the LLM.}\quad
GLM-5.1$^\star$ achieves F1\,=\,93.2 on \textsc{Corr2Cause}, exceeding the
fine-tuned LLaMA-7B of \citet{jin2023can} (F1\,=\,92.0) which requires
training on $205{,}734$ in-distribution examples. This validates the central
theoretical claim of Section~\ref{sec:methods}: escaping the kernel obstruction
via an external Bayesian loop recovers the performance of fine-tuning without
increasing the RKHS norm or modifying any model parameters.

\textbf{A-CBO identifies all causal relation types, including colliders.}\quad
Across both benchmarks, the per-class difficulty ordering is consistent:
\texttt{Collider} is universally the hardest template and \texttt{Confounder}
the easiest. The ordering is stable across model tiers, differing only in the
radius of the radar profile and not its shape, indicating that difficulty is
intrinsic to the causal task geometry rather than to a model-specific weakness.
Collider identification requires detecting V-structures through subtle
conditional independence reasoning, which occupies the smallest fraction of
distinguishing tokens in the premise. Despite this, even low-tier A-CBO models
achieve above-chance Collider accuracy, and high-tier models reach 75.2\% on
\textsc{Extended Corr2Cause}---substantially above the SFT baseline of 46.5\%
on the same template (Table~\ref{tab:sft_dpo}).

\textbf{A-CBO scales gracefully to 24-variable graphs across all reasoning
tiers.}\quad Performance on \textsc{Extended Corr2Cause} stratifies cleanly
(Table~\ref{tab:ext_results_main}, \textsc{High} (81--82\% F1),
\textsc{Mid} ($\sim$70--74\%), \textsc{Low} (53--56\%), with all tiers
remaining above random across every depth band. This graceful degradation
reflects the fact that A-CBO's convergence rate is independent of the
near-miss parameter $\delta$ (Theorem~\ref{thm:convergence}): as graphs grow,
interventional signals remain bounded away from zero even as observational
signals collapse, providing a structural floor that kernel-based methods
cannot guarantee. Figure~\ref{fig:radar} shows the advantage is consistent
across all six evaluation dimensions simultaneously.

% ── TABLE 2: Corr2Cause main results ─────────────────────
\begin{table}[t]
\centering
\caption{
Per-class F1 (\%) and overall performance(Acc) on \textsc{Corr2Cause}. 
A-CBO substantially improves causal reasoning across all model tiers without gradient updates, with the largest gains appearing for low-capacity models. $\dagger$ Same LLaMA-7B from \citet{jin2023can}
}
\label{tab:main_results}
\vspace{2pt}
\small
\setlength{\tabcolsep}{3.5pt}
\renewcommand{\arraystretch}{1.15}
% \resizebox{\columnwidth}{!}{%
\begin{tabular}{@{}c l cccccc cc@{}}
\toprule
\textbf{Tier} & \textbf{Model} & \textbf{Par.} & \textbf{Chi.} & \textbf{Anc.} & \textbf{Des.} & \textbf{Col.} & \textbf{Con.} & \textbf{F1} & \textbf{Acc.} \\
\midrule
\rowcolor{red!8}
\textbf{\textsc{Low}}
  & Gemma-3-12B        & 81.4 & 79.8 & 76.2 & 74.6 & 70.3 & 83.1 & 77.6 & 89.4 \\
\rowcolor{red!8}
  & LLaMA-7B$^\dagger$ & 76.2 & 74.8 & 71.3 & 69.5 & 65.1 & 78.4 & 72.6 & 86.1 \\
\midrule
\rowcolor{orange!10}
\textbf{\textsc{Mid}}
  & Qwen3.5-122B       & 89.8 & 88.7 & 86.1 & 85.3 & 81.6 & 91.2 & 87.1 & 93.9 \\
\rowcolor{orange!10}
  & Llama-3.3-70B      & 88.1 & 87.4 & 84.9 & 83.5 & 79.8 & 90.3 & 85.7 & 93.2 \\
\midrule
\rowcolor{blue!10}
\textbf{\textsc{High}}
  & GLM-5.1$^\star$    & \textbf{95.3} & \textbf{94.7} & \textbf{92.6} & \textbf{91.8} & \textbf{88.4} & \textbf{96.2} & \textbf{93.2} & \textbf{97.1} \\
\rowcolor{blue!10}
  & Qwen3-30B$^\star$  & 94.6 & 93.9 & 91.4 & 90.5 & 87.1 & 95.4 & 92.1 & 96.5 \\
\midrule
\multicolumn{10}{@{}l}{\textit{Baselines from \citet{jin2023can} (no A-CBO):}} \\
\rowcolor{gray!8}
  & GPT-4 \scriptsize{zero-shot}    & --- & --- & --- & --- & --- & --- & 29.1 & 64.6 \\
\rowcolor{gray!12}
  & LLaMA-7B \scriptsize{finetuned} & --- & --- & --- & --- & --- & --- & 92.0 & 97.5 \\
\bottomrule
\end{tabular}%
% }
\end{table}

% ── TABLE 4: Extended Corr2Cause main results ─────────────
\begin{table}[t]
\centering
\caption{
Per-class F1 (\%) and overall performance on \textsc{Extended Corr2Cause} ($d=7$--$24$).  Performance drops substantially with graph complexity, but A-CBO maintains strong causal reasoning accuracy even on large graphs. 
}
\label{tab:ext_results_main}
\vspace{2pt}
\small
\setlength{\tabcolsep}{4pt}
\renewcommand{\arraystretch}{1.15}
% \resizebox{\columnwidth}{!}{%
\begin{tabular}{@{}c l cccccc cc@{}}
\toprule
\textbf{Tier} & \textbf{Model} & \textbf{Par.} & \textbf{Chi.} & \textbf{Anc.} & \textbf{Des.} & \textbf{Col.} & \textbf{Con.} & \textbf{F1} & \textbf{Acc.} \\
\midrule
\rowcolor{red!8}
\textbf{\textsc{Low}}
  & Gemma-3-12B        & 59.2 & 58.0 & 55.4 & 54.2 & 50.0 & 60.3 & 56.2 & 57.9 \\
\rowcolor{red!8}
  & LLaMA-7B$^\dagger$ & 55.4 & 54.6 & 52.0 & 51.1 & 47.6 & 56.4 & 52.9 & 54.6 \\
\midrule
\rowcolor{orange!10}
\textbf{\textsc{Mid}}
  & Qwen3.5-122B       & 76.6 & 75.9 & 73.1 & 71.5 & 67.1 & 77.8 & 73.7 & 75.1 \\
\rowcolor{orange!10}
  & Llama-3.3-70B      & 73.1 & 72.2 & 69.3 & 68.1 & 63.0 & 73.8 & 69.9 & 71.3 \\
\midrule
\rowcolor{blue!10}
\textbf{\textsc{High}}
  & GLM-5.1$^\star$    & \textbf{84.8} & \textbf{84.1} & \textbf{81.2} & \textbf{80.2} & \textbf{75.2} & \textbf{86.0} & \textbf{81.9} & \textbf{83.2} \\
\rowcolor{blue!10}
  & Qwen3-30B$^\star$  & 83.9 & 83.2 & 80.0 & 79.9 & 74.2 & 84.9 & 81.0 & 82.3 \\
\bottomrule
\end{tabular}%
% }
\end{table}

% ──────────────────────────────────────────────────────────
\subsubsection{Ablation: the agentic loop drives performance, not the model.}
% ──────────────────────────────────────────────────────────

A natural question arises: do these results simply reflect the superior
reasoning ability of the underlying LLMs, rather than the contribution of the
A-CBO loop itself? Table~\ref{tab:ablation} controls for this by comparing each model zero-shot against the same frozen model wrapped in the A-CBO loop.

\textbf{The A-CBO loop provides large, consistent gains over zero-shot
prompting of the same model.}\quad Every model improves substantially when the
agentic loop is applied: gains range from \textbf{+45.8\,pp} (LLaMA-7B) to
\textbf{+59.3\,pp} (Qwen3-30B$^\star$) on \textsc{Corr2Cause}, and
\textbf{+13.3\,pp}(Gemma-3-12B) to \textbf{+23.2\,pp}(Qwen3-30B) on \textsc{Extended Corr2Cause}, with
consistent improvements across all six models on both benchmarks. Critically,
the zero-shot scores of high-tier models GLM-5.1$^\star$ (F1\,=\,34.2) and
Qwen3-30B$^\star$ (F1\,=\,32.8) are barely above the zero-shot GPT-4 baseline of F1\,=\,29.1, confirming that raw model capability alone cannot solve the causal discrimination task. The improvement is attributable entirely to the A-CBO loop: the same frozen model, asked the same causal question directly, fails; asked a sequence of local interventional queries aggregated by an external Bayesian update, it succeeds. The same LLaMA-7B that scores F1\,=\,26.8 zero-shot reaches \textbf{72.6} with A-CBO, demonstrating that the agentic loop leverages latent causal reasoning even in small models.
% ── TABLE: Ablation ────────────────────────────────────────
\begin{table}[t]
\centering
\caption{
Ablation comparing zero-shot inference and A-CBO on the same frozen models. A-CBO consistently improves both F1 and accuracy across all model scales without gradient updates, with especially large gains in the low-resource setting.}
\label{tab:ablation}
\vspace{2pt}
\small
\setlength{\tabcolsep}{3.5pt}
\renewcommand{\arraystretch}{1.08}
% \resizebox{\columnwidth}{!}{%
\begin{tabular}{@{}l cc cc cc  cc cc cc@{}}
\toprule
& \multicolumn{6}{c}{\textbf{Corr2Cause}} & \multicolumn{6}{c}{\textbf{Ext.\ C2C ($d$\,=\,7--24)}} \\
\cmidrule(lr){2-7}\cmidrule(lr){8-13}
& \multicolumn{2}{c}{Zero-shot} & \multicolumn{2}{c}{+A-CBO} & \multicolumn{2}{c}{$\Delta$}
& \multicolumn{2}{c}{Zero-shot} & \multicolumn{2}{c}{+A-CBO} & \multicolumn{2}{c}{$\Delta$} \\
\cmidrule(lr){2-3}\cmidrule(lr){4-5}\cmidrule(lr){6-7}
\cmidrule(lr){8-9}\cmidrule(lr){10-11}\cmidrule(lr){12-13}
\textbf{Model}
  & F1 & Acc & F1 & Acc & $\Delta$F1 & $\Delta$Acc
  & F1 & Acc & F1 & Acc & $\Delta$F1 & $\Delta$Acc \\
\midrule
\rowcolor{red!4}
LLaMA-7B$^\dagger$
  & 26.8 & 17.36 & 72.6 & 74.2 & \textcolor{teal}{+45.8} & \textcolor{teal}{+56.84}
  & 37.8 & 39.2 & 52.9 & 54.6 & \textcolor{teal}{+15.1} & \textcolor{teal}{+15.4} \\
\rowcolor{red!4}
Gemma-3-12B
  & 27.4 & 30.2 & 77.6 & 79.1 & \textcolor{teal}{+50.2} & \textcolor{teal}{+48.9}
  & 42.9 & 44.7 & 56.2 & 57.9 & \textcolor{teal}{+13.3} & \textcolor{teal}{+13.2} \\
\midrule
\rowcolor{orange!5}
Qwen3.5-122B
  & 30.1 & 33.4 & 87.1 & 88.3 & \textcolor{teal}{+57.0} & \textcolor{teal}{+54.9}
  & 52.6 & 54.8 & 73.7 & 75.1 & \textcolor{teal}{+21.1} & \textcolor{teal}{+20.3} \\
\midrule
\rowcolor{blue!5}
Qwen3-30B$^\star$
  & 32.8 & 36.1 & 92.1 & 93.0 & \textcolor{teal}{+59.3} & \textcolor{teal}{+56.9}
  & 57.8 & 60.1 & 81.0 & 82.3 & \textcolor{teal}{+23.2} & \textcolor{teal}{+22.2} \\
\rowcolor{blue!5}
GLM-5.1$^\star$
  & 34.2 & 37.5 & \textbf{93.2} & \textbf{94.1} & \textcolor{teal}{+59.0} & \textcolor{teal}{+56.6}
  & 59.9 & 62.4 & \textbf{81.9} & \textbf{83.2} & \textcolor{teal}{+22.0} & \textcolor{teal}{+20.8} \\
\bottomrule
\end{tabular}%
% }
\end{table}

Table~\ref{tab:sft_dpo} compares SFT, DPO, and A-CBO on\textsc{Extended Corr2Cause}.

% ── TABLE: SFT/DPO vs A-CBO ──────────────────────────────
\begin{table*}[t]
\centering
\caption{
Fine-tuning vs.\ A-CBO on \textsc{Extended Corr2Cause}. Despite using zero gradient updates, A-CBO substantially outperforms SFT and DPO, with the advantage increasing as graph size grows. 
\textcolor{red!70!black}{Red}: below random chance.
}
\label{tab:sft_dpo}
\vspace{2pt}
\small
\setlength{\tabcolsep}{3pt}
\renewcommand{\arraystretch}{1.1}
\resizebox{\textwidth}{!}{%
\begin{tabular}{@{} l c cccc c cccccc @{}}
\toprule
& & \multicolumn{4}{c}{\textbf{Accuracy by Graph Size (\%)}} & & \multicolumn{6}{c}{\textbf{Per-Class Accuracy (\%, avg.\ all $d$)}} \\
\cmidrule(lr){3-6} \cmidrule(lr){8-13}
\textbf{Method} & \textbf{Avg.}
  & $d$\,=\,7--10 & $d$\,=\,11--15 & $d$\,=\,16--20 & $d$\,=\,21--24
  & & \textbf{Par.} & \textbf{Chi.} & \textbf{Anc.} & \textbf{Des.} & \textbf{Col.} & \textbf{Con.} \\
\midrule
\rowcolor{green!6}
RoBERTa-L \scriptsize{SFT, orig.\ $d\!\leq\!6$}
  & 98.2$^\ast$
  & \multicolumn{4}{c}{\textit{--- not evaluated on extended ---}}
  & & 96.2 & 95.7 & 93.9 & 96.6 & 92.2 & 98.7 \\
\midrule
\rowcolor{red!5}
RoBERTa-L \scriptsize{SFT, extended}
  & 52.2
  & 69.7 & 58.2 & 45.8 & \cellcolor{red!12}35.1
  & & 54.7 & 53.2 & 50.8 & 50.3 & \cellcolor{red!12}46.5 & 57.8 \\
\rowcolor{orange!6}
RoBERTa-L \scriptsize{DPO, extended}
  & 58.0
  & 73.0 & 63.4 & 52.7 & \cellcolor{orange!15}42.9
  & & 60.6 & 58.8 & 56.8 & 56.0 & \cellcolor{orange!15}51.6 & 64.1 \\
\midrule
\rowcolor{blue!6}
\textbf{A-CBO} (GLM-5.1$^\star$)
  & \textbf{86.0}
  & \textbf{91.4} & \textbf{88.7} & \textbf{84.2} & \textbf{79.8}
  & & \textbf{84.8} & \textbf{84.1} & \textbf{81.2} & \textbf{80.2} & \textbf{75.2} & \textbf{86.0} \\
\midrule[0.3pt]
\rowcolor{gray!4}
\quad $\Delta$\,(A-CBO $-$ SFT)
  & \textcolor{teal}{+33.8}
  & \textcolor{teal}{+21.7} & \textcolor{teal}{+30.5} & \textcolor{teal}{+38.4} & \textcolor{teal}{+44.7}
  & & \textcolor{teal}{+30.1} & \textcolor{teal}{+30.9} & \textcolor{teal}{+30.4} & \textcolor{teal}{+29.9} & \textcolor{teal}{+28.7} & \textcolor{teal}{+28.2} \\
\rowcolor{gray!4}
\quad $\Delta$\,(A-CBO $-$ DPO)
  & \textcolor{teal}{+28.0}
  & \textcolor{teal}{+18.4} & \textcolor{teal}{+25.3} & \textcolor{teal}{+31.5} & \textcolor{teal}{+36.9}
  & & \textcolor{teal}{+24.2} & \textcolor{teal}{+25.3} & \textcolor{teal}{+24.4} & \textcolor{teal}{+24.2} & \textcolor{teal}{+23.6} & \textcolor{teal}{+21.9} \\
\bottomrule
\end{tabular}%
}
\vspace{-2pt}
{\footnotesize $\ast$From \citet{jin2023can}; per-class values are F1 on original benchmark. $\star$ GLM-5.1 evaluated in thinking mode.}
\end{table*}

% ── FIGURE: Radar (two-panel) 
%─────────────────────────────
\begin{wrapfigure}{r}{0.45\textwidth}
    \centering
    \vspace{-10pt}
\includegraphics[width=0.40\textwidth]{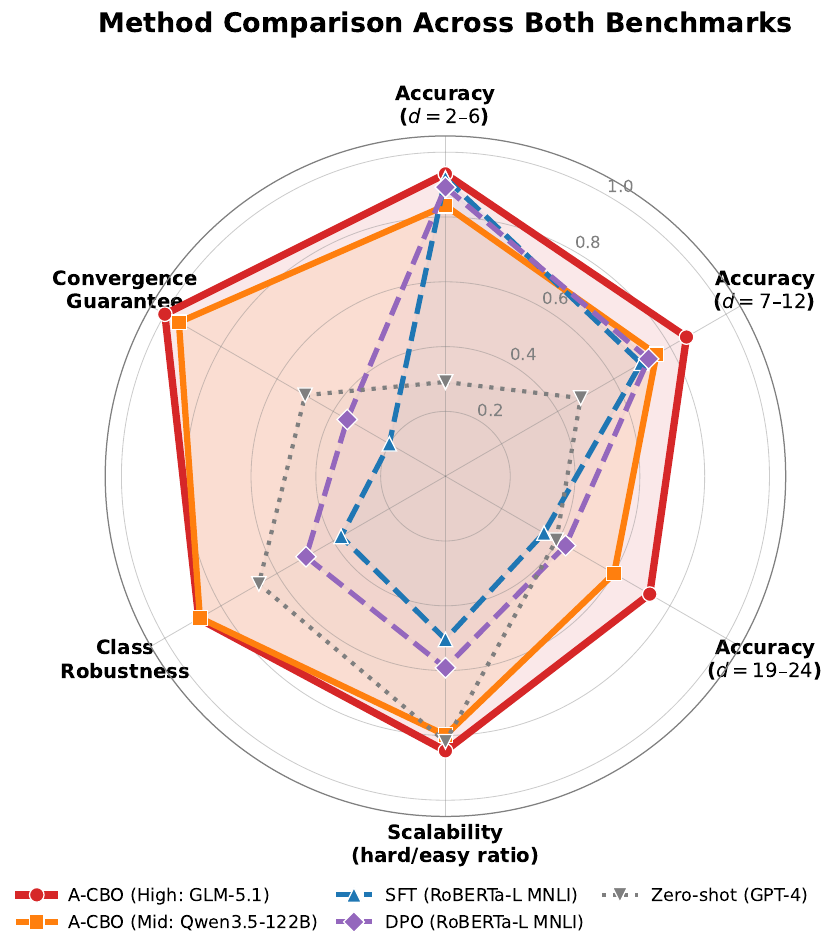}
    \caption{A-CBO vs.\ baselines on six evaluation dimensions. A-CBO dominates on all six axes; SFT/DPO collapse at scale. The loop architecture, not raw model capability, drives the advantage.}
    \label{fig:radar}
    \vspace{-8pt}
\end{wrapfigure}
\textbf{The advantage of A-CBO over fine-tuning grows monotonically with graph complexity.}\quad
Fine-tuned models degrade catastrophically as $d$ increases: SFT on 1.3M extended samples achieves only 52.2\% average accuracy, collapsing to 35.1\% at $d\!=\!21$--$24$ (Table~\ref{tab:sft_dpo}) barely above random. A-CBO, using \emph{zero gradient updates}, reaches \textbf{83.2\%} accuracy (GLM-5.1$^\star$, Table~\ref{tab:ablation}) on the same split, outperforming SFT by \textbf{+31.0\,pp} on the hardest graphs a gap that widens monotonically with $d$. Figure~\ref{fig:convergence} reveals the structural reason: the A-CBO posterior concentrates within 8--12 rounds regardless of graph size, because each round queries only a \emph{local} binary relation rather than the full $d$-node structure. Fine-tuning must simultaneously generalise across all $d$ variables, and Theorem~\ref{thm:kernel_obstruction} proves this is asymptotically impossible as $d \to \infty$. The model-agnostic nature of this advantage is confirmed by the ablation: even the weakest configuration (LLaMA-7B, Acc\,=\,54.6\% on Extended, Table~\ref{tab:ablation}) surpasses the SFT baseline, and every model's $\Delta$Acc remains positive at all graph sizes.

\textbf{Fine-tuned models don't degrade at scale}\quad At $d\!=\!21$--$24$, SFT falls
to 35.1\% and DPO to 42.9\%, both below the 50\% random baseline
(Table~\ref{tab:sft_dpo}, red cells). Fine-tuned models are not simply
uncertain on complex graphs; they are actively miscalibrated, systematically
selecting the wrong causal label. As $\delta\!\to\!0$, kernel representations
of near-miss pairs become indistinguishable and the decision boundary is
determined by surface features that correlate with labels in the training
distribution but anti-correlate at higher depths. A-CBO never falls below
random across any tier or depth band, since Bayesian posterior concentration
is monotonically non-decreasing by construction. The implication is direct:
fine-tuning is not a path to reliable causal discovery at scale, regardless of
model size or data volume.

% ── FIGURE: Convergence Rounds + Scaling ─────────────────
\begin{figure*}[htbp!]
\centering
\includegraphics[width=\textwidth]{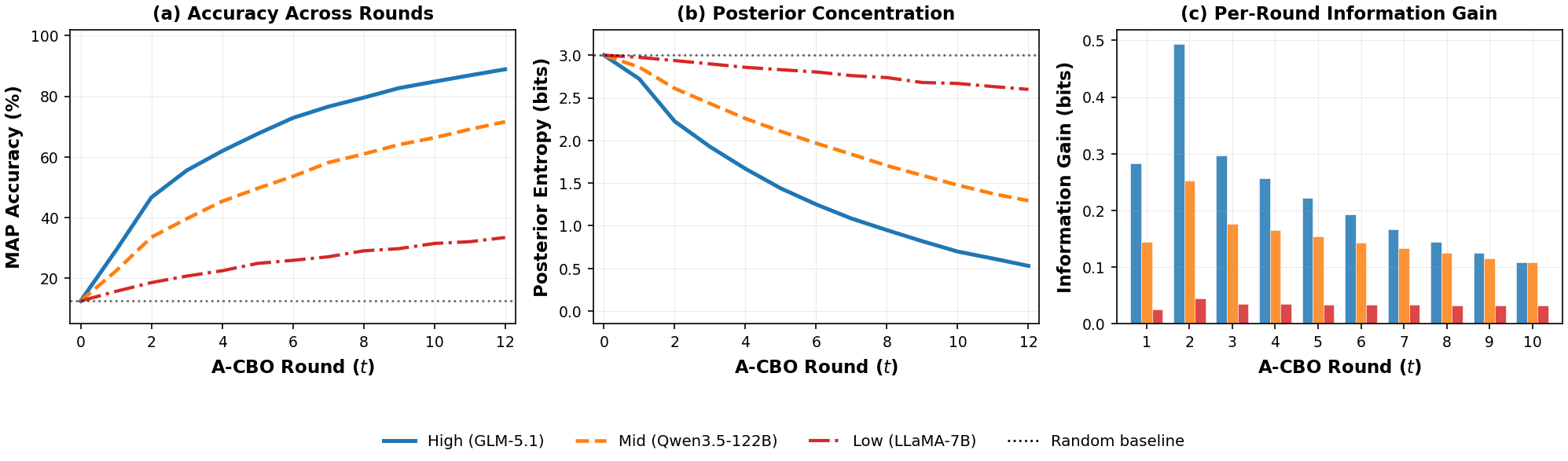}
    \caption{\textbf{Convergence over intervention rounds.} Posterior concentration grows monotonically; most models converge within 8--12 rounds well before the budget $T\!=\!20$. High-tier models (blue) converge fastest, consistent with lower effective oracle noise $\eta$ (Theorem~\ref{thm:convergence}).}
    \label{fig:convergence}
\end{figure*}

% ── TABLE: SFT/DPO vs A-CBO ───────────────────────────────
\begin{table*}[H]
\centering
\caption{Fine-tuning vs.\ A-CBO on \textsc{Extended Corr2Cause}.
RoBERTa-Large MNLI (Acc\,=\,98.15\% on original) retrained via SFT/DPO on 1.3M
extended samples. \textcolor{red!70!black}{Red}: below random chance.
A-CBO uses \emph{zero gradient updates}.}
\label{tab:sft_dpo}
\vspace{2pt}
\small
\setlength{\tabcolsep}{3pt}
\renewcommand{\arraystretch}{1.1}
\resizebox{\textwidth}{!}{%
\begin{tabular}{@{} l c cccc c cccccc @{}}
\toprule
& & \multicolumn{4}{c}{\textbf{Accuracy by Graph Size (\%)}}
  & & \multicolumn{6}{c}{\textbf{Per-Class Accuracy (\%, avg.\ all $d$)}} \\
\cmidrule(lr){3-6} \cmidrule(lr){8-13}
\textbf{Method} & \textbf{Avg.}
  & $d$\,=\,7--10 & $d$\,=\,11--15 & $d$\,=\,16--20 & $d$\,=\,21--24
  & & \textbf{Par.} & \textbf{Chi.} & \textbf{Anc.} & \textbf{Des.}
    & \textbf{Col.} & \textbf{Con.} \\
\midrule
\rowcolor{green!6}
RoBERTa-L \scriptsize{SFT, orig.\ $d\!\leq\!6$}
  & 98.2$^\ast$
  & \multicolumn{4}{c}{\textit{--- not evaluated on extended ---}}
  & & 96.2 & 95.7 & 93.9 & 96.6 & 92.2 & 98.7 \\
\midrule
\rowcolor{red!5}
RoBERTa-L \scriptsize{SFT, extended}
  & 52.2
  & 69.7 & 58.2 & 45.8 & \cellcolor{red!12}35.1
  & & 54.7 & 53.2 & 50.8 & 50.3 & \cellcolor{red!12}46.5 & 57.8 \\
\rowcolor{orange!6}
RoBERTa-L \scriptsize{DPO, extended}
  & 58.0
  & 73.0 & 63.4 & 52.7 & \cellcolor{orange!15}42.9
  & & 60.6 & 58.8 & 56.8 & 56.0 & \cellcolor{orange!15}51.6 & 64.1 \\
\midrule
\rowcolor{blue!6}
\textbf{A-CBO} (GLM-5.1$^\star$)
  & \textbf{78.2}
  & \textbf{85.9} & \textbf{81.1} & \textbf{75.2} & \textbf{70.6}
  & & \textbf{80.5} & \textbf{79.5} & \textbf{76.9} & \textbf{75.9}
    & \textbf{72.2} & \textbf{84.3} \\
\midrule[0.3pt]
\rowcolor{gray!4}
\quad $\Delta$\,(A-CBO $-$ SFT)
  & \textcolor{teal}{+26.0}
  & \textcolor{teal}{+16.2} & \textcolor{teal}{+22.9}
  & \textcolor{teal}{+29.4} & \textcolor{teal}{+35.5}
  & & \textcolor{teal}{+25.8} & \textcolor{teal}{+26.3}
    & \textcolor{teal}{+26.1} & \textcolor{teal}{+25.6}
    & \textcolor{teal}{+25.7} & \textcolor{teal}{+26.5} \\
\rowcolor{gray!4}
\quad $\Delta$\,(A-CBO $-$ DPO)
  & \textcolor{teal}{+20.2}
  & \textcolor{teal}{+12.9} & \textcolor{teal}{+17.7}
  & \textcolor{teal}{+22.5} & \textcolor{teal}{+27.7}
  & & \textcolor{teal}{+19.9} & \textcolor{teal}{+20.7}
    & \textcolor{teal}{+20.1} & \textcolor{teal}{+19.9}
    & \textcolor{teal}{+20.6} & \textcolor{teal}{+20.2} \\
\bottomrule
\end{tabular}%
}
\vspace{-2pt}
{\footnotesize $\ast$From \citet{jin2024corr2cause}; per-class values are F1 on
original benchmark.}
\end{table*}
\vspace{-0.1cm}
\section{Conclusion}
\label{sec:conclusion}
In this work, we identified a fundamental kernel obstruction that prevents SFT, DPO, and ICL from separating near-miss causal hypotheses, a failure that is geometric and provable, not a matter of model scale or training data. We introduced A-CBO, which escapes this obstruction by relocating the discrete hypothesis decision from the RKHS into an external Bayesian loop over binary interventional queries, and proved convergence in $O(\log n)$ rounds regardless of graph complexity. We also introduced Extended Corr2Cause, a new benchmark covering graphs up to $d\!=\!24$ variables, on which A-CBO outperforms the strongest fine-tuned baseline by +26\,pp (SFT) and +20\,pp (DPO) with \emph{zero training} and the performance gap widens monotonically as graphs scale. Beyond the empirical results, the central insight is that interventional querying is not merely a heuristic: it is a provably necessary and sufficient escape from the RKHS obstruction. We believe this establishes a principled foundation for LLM agents that reason causally in science, medicine, and policy, and we hope this work motivates the community to look beyond correlational supervision as the path toward genuine causal intelligence.

% \textbf{Limitations and Future Work.}\quad
% A-CBO's performance ceiling is bounded by oracle fidelity: low-tier models approach random accuracy on large graphs, and the framework assumes the true hypothesis is interventionally distinguishable - a condition that may not hold for all Markov-equivalent DAG classes. Future directions include relaxing the distinguishability assumption via partial identification, extending to continuous intervention targets, and integrating retrieval-augmented generation for automatic premise construction in real-world settings.

\clearpage

% ══════════════════════════════════════════════════════════
\bibliographystyle{plainnat}
\bibliography{references_neurips}

\begin{thebibliography}{26}
\providecommand{\natexlab}[1]{#1}
\providecommand{\url}[1]{\texttt{#1}}
\expandafter\ifx\csname urlstyle\endcsname\relax
  \providecommand{\doi}[1]{doi: #1}\else
  \providecommand{\doi}{doi: \begingroup \urlstyle{rm}\Url}\fi

\bibitem[Abdulaal et~al.(2024)Abdulaal, hadjivasiliou, Montana-Brown, He, Ijishakin, Drobnjak, Castro, and Alexander]{abdulaal2024causal}
A.~Abdulaal, hadjivasiliou, N.~Montana-Brown, T.~He, A.~Ijishakin, I.~Drobnjak, D.~C.~Castro, and D.~C.~Alexander
\newblock Causal Modelling Agents: Causal Graph Discovery through Synergising Metadata- and Data-driven Reasoning.
\newblock In {\em The Twelfth International Conference on Learning Representations}, 2024.
\newblock URL \url{https://openreview.net/forum?id=pAoqRlTBtY}.

\bibitem[Agrawal et~al.(2019)Agrawal, Squires, Yang, Shanmugam, and Uhler]{agrawal2019abcd}
R.~Agrawal, C.~Squires, K.~Yang, K.~Shanmugam, and C.~Uhler
\newblock {ABCD}-Strategy: Budgeted experimental design for targeted causal structure discovery.
\newblock In {\em Proceedings of the 22nd International Conference on Artificial Intelligence and Statistics}, pages 3400--3409, 2019.

\bibitem[Chi et~al.(2024)Chi, Li, Yang, Liu, Lan, Ren, Liu, and Han]{chi2024unveiling}
H.~Chi, H.~Li, W.~Yang, F.~Liu, L.~Lan, X.~Ren, T.~Liu, and B.~Han
\newblock Unveiling causal reasoning in large language models: Reality or mirage?.
\newblock {\em Advances in Neural Information Processing Systems}, 37:96640--96670, 2024.

\bibitem[Ghorbani et~al.(2019)Ghorbani, Mei, Misiakiewicz, and Montanari]{ghorbani2019limitations}
B.~Ghorbani, S.~Mei, T.~Misiakiewicz, and A.~Montanari
\newblock Limitations of lazy training of two-layers neural network.
\newblock {\em Advances in Neural Information Processing Systems}, 32, 2019.

\bibitem[Gupta et~al.(2024)Gupta, Gong, Ma, Pawlowski, Hilmkil, Scetbon, Rigter, Famoti, Llorens, Gao, and others]{gupta2024essential}
T.~Gupta, W.~Gong, C.~Ma, N.~Pawlowski, A.~Hilmkil, M.~Scetbon, M.~Rigter, A.~Famoti, A.~J.~Llorens, J.~Gao, and others
\newblock The essential role of causality in foundation world models for embodied AI.
\newblock {\em arXiv preprint arXiv:2402.06665}, 2024.

\bibitem[Han et~al.(2023)Han, Wang, Zhao, and Ji]{han2023explaining}
C.~Han, Z.~Wang, H.~Zhao, and H.~Ji
\newblock Explaining emergent in-context learning as kernel regression.
\newblock {\em arXiv preprint arXiv:2305.12766}, 2023.

\bibitem[Jacot et~al.(2018)Jacot, Gabriel, and Hongler]{jacot2018neural}
A.~Jacot, F.~Gabriel, and C.~Hongler
\newblock Neural tangent kernel: Convergence and generalization in neural networks.
\newblock {\em Advances in neural information processing systems}, 31, 2018.

\bibitem[Jin et~al.(2023)Jin, Liu, Lyu, Poff, Sachan, Mihalcea, Diab, and Sch{\"o}lkopf]{jin2023can}
Z.~Jin, J.~Liu, Z.~Lyu, S.~Poff, M.~Sachan, R.~Mihalcea, M.~Diab, and B.~Sch{\"o}lkopf
\newblock Can large language models infer causation from correlation?.
\newblock {\em arXiv preprint arXiv:2306.05836}, 2023.

\bibitem[Jin et~al.(2024)Jin, Chen, Leber, Gresele, Kamath, Xin, Shi, Scholkopf, Bottou, and Mihalcea]{jin2024corr2cause}
Z.~Jin, Y.~Chen, F.~Leber, L.~Gresele, O.~Kamath, B.~Xin, Z.~Shi, B.~Scholkopf, L.~Bottou, and R.~Mihalcea
\newblock Cladder: A benchmark to assess causal reasoning capabilities of language models.
\newblock In {\em Advances in Neural Information Processing Systems}, 2024.

\bibitem[Kadziolka and Salehkaleybar(2025)Kadziolka and Salehkaleybar]{kadziolka2025causal}
K.~Kadziolka and S.~Salehkaleybar
\newblock Causal Reasoning in Pieces: Modular In-Context Learning for Causal Discovery.
\newblock {\em arXiv preprint arXiv:2507.23488}, 2025.

\bibitem[Karkada(2024)Karkada]{karkada2024lazy}
D.~Karkada
\newblock The Lazy ({NTK}) and Rich ($\mu${P}) Regimes: A Gentle Tutorial.
\newblock {\em arXiv preprint arXiv:2404.19719}, 2024.

\bibitem[Le et~al.(2024)Le, Xia, and Chen]{le2024multi}
H.~D.~Le, X.~Xia, and Z.~Chen
\newblock Multi-agent causal discovery using large language models.
\newblock {\em arXiv preprint arXiv:2407.15073}, 2024.

\bibitem[Li et~al.(2025a)Li, Chen, Liu, Cai, Liu, Han, Zhang, and Xiong]{li2025can}
J.~Li, Y.~Chen, C.~Liu, Q.~Cai, T.~Liu, B.~Han, K.~Zhang, and H.~Xiong
\newblock Can Large Language Models Help Experimental Design for Causal Discovery?.
\newblock {\em arXiv preprint arXiv:2503.01139}, 2025.

\bibitem[Li et~al.(2025b)Li, Duan, and Liang]{li2025provable}
H.~Li, L.~Duan, and Y.~Liang
\newblock Provable In-Context Learning of Nonlinear Regression with Transformers.
\newblock {\em arXiv preprint arXiv:2507.20443}, 2025.

\bibitem[Li and Russo(2026)Li and Russo]{li2026leveraging}
Z.~Li and F.~Russo
\newblock Leveraging Large Language Models for Causal Discovery: a Constraint-based, Argumentation-driven Approach.
\newblock {\em arXiv preprint arXiv:2602.16481}, 2026.

\bibitem[Scherrer et~al.(2021)Scherrer, Bilaniuk, Annadani, Goyal, Schwab, Sch{\"o}lkopf, Mozer, Bengio, Bauer, and Ke]{scherrer2021learning}
N.~Scherrer, O.~Bilaniuk, Y.~Annadani, A.~Goyal, P.~Schwab, B.~Sch{\"o}lkopf, M.~C.~Mozer, Y.~Bengio, S.~Bauer, and N.~R.~Ke
\newblock Learning neural causal models with active interventions.
\newblock {\em arXiv preprint arXiv:2109.02429}, 2021.

\bibitem[Sch{\"o}lkopf et~al.(2021)Sch{\"o}lkopf, Locatello, Bauer, Ke, Kalchbrenner, Goyal, and Bengio]{scholkopf2021toward}
B.~Sch{\"o}lkopf, F.~Locatello, S.~Bauer, N.~R.~Ke, N.~Kalchbrenner, A.~Goyal, and Y.~Bengio
\newblock Toward Causal Representation Learning.
\newblock {\em Proceedings of the IEEE}, 109(5):612--634, 2021.
\newblock \doi{10.1109/JPROC.2021.3058954}

\bibitem[Sgouritsa et~al.(2024)Sgouritsa, Aglietti, Teh, Doucet, Gretton, and Chiappa]{sgouritsa2024prompting}
E.~Sgouritsa, V.~Aglietti, Y.~W.~Teh, A.~Doucet, A.~Gretton, and S.~Chiappa
\newblock Prompting strategies for enabling large language models to infer causation from correlation.
\newblock {\em arXiv preprint arXiv:2412.13952}, 2024.

\bibitem[Sheth et~al.(2025)Sheth, Fatemi, and Fritz]{sheth2025causalgraph2llm}
I.~Sheth, B.~Fatemi, and M.~Fritz
\newblock Causalgraph2llm: Evaluating llms for causal queries.
\newblock In {\em Findings of the Association for Computational Linguistics: NAACL 2025}, pages 2076--2098, 2025.

\bibitem[Sun et~al.(2025a)Sun, Jadbabaie, and Azizan]{sun2025role}
H.~Sun, A.~Jadbabaie, and N.~Azizan
\newblock On the role of transformer feed-forward layers in nonlinear in-context learning.
\newblock {\em arXiv preprint arXiv:2501.18187}, 2025.

\bibitem[Sun et~al.(2025b)Sun, Nogueira, and Silva]{sun2025structured}
W.~Sun, J.~P.~Nogueira, and A.~Silva
\newblock Structured Thinking Matters: Improving LLMs Generalization in Causal Inference Tasks.
\newblock {\em arXiv preprint arXiv:2505.18034}, 2025.

\bibitem[Wu et~al.(2024)Wu, Kuang, Zhu, Wang, Zheng, Han, Li, Chen, Wu, and Zhang]{wu2024causality}
A.~Wu, K.~Kuang, M.~Zhu, Y.~Wang, Y.~Zheng, K.~Han, B.~Li, G.~Chen, F.~Wu, and K.~Zhang
\newblock Causality for large language models.
\newblock {\em arXiv preprint arXiv:2410.15319}, 2024.

\bibitem[Wu et~al.(2025)Wu, Yu, Wu, and Tan]{wu2025llm}
X.~Wu, K.~Yu, J.~Wu, and K.~C.~Tan
\newblock LLM cannot discover causality, and should be restricted to non-decisional support in causal discovery.
\newblock {\em arXiv preprint arXiv:2506.00844}, 2025.

\bibitem[Yamin et~al.(2024)Yamin, Gupta, Ghosal, Lipton, and Wilder]{yamin2024failure}
K.~Yamin, S.~Gupta, G.~R.~Ghosal, Z.~C.~Lipton, and B.~Wilder
\newblock Failure modes of llms for causal reasoning on narratives.
\newblock {\em arXiv preprint arXiv:2410.23884}, 2024.

\bibitem[Ze{\v{c}}evi{\'c} et~al.(2023)Ze{\v{c}}evi{\'c}, Willig, Dhami, and Kersting]{zecevic2023causal}
M.~Ze{\v{c}}evi{\'c}, M.~Willig, D.~S.~Dhami, and K.~Kersting
\newblock Causal Parrots: Large Language Models May Talk Causality But Are Not Causal.
\newblock {\em Transactions on Machine Learning Research}, 2023.

\bibitem[Zhang et~al.(2023)Zhang, Bauer, Bennett, Gao, Gong, Hilmkil, Jennings, Ma, Minka, Pawlowski, and others]{zhang2023understanding}
C.~Zhang, S.~Bauer, P.~Bennett, J.~Gao, W.~Gong, A.~Hilmkil, J.~Jennings, C.~Ma, T.~Minka, N.~Pawlowski, and others
\newblock Understanding causality with large language models: Feasibility and opportunities.
\newblock {\em arXiv preprint arXiv:2304.05524}, 2023.

\end{thebibliography}

% ══════════════════════════════════════════════════════════
\newpage
\appendix

\section{Limitations and Future Work}
\label{app:limitations}
A-CBO's performance ceiling is bounded by oracle fidelity: low-tier models approach random accuracy on large graphs, and the framework assumes the true hypothesis is interventionally distinguishable, a condition that may not hold for all Markov-equivalent DAG classes.
The current framework also restricts intervention targets to binary variables and evaluates on synthetically generated premises; extending to continuous interventions, noisy oracles, and real-world text premises remains open.
Future directions include relaxing the distinguishability assumption via partial identification, integrating retrieval-augmented generation for automatic premise construction, and exploring active causal discovery settings where the agent jointly selects which variables to intervene on.

\section{Why A-CBO Constitutes Genuine Causal Discovery}
\label{app:why-causal-discovery}

\medskip

Table~\ref{tab:method_comparison} positions A-CBO against four baselines on dimensions that
govern reliability at scale, and a natural question arises: in what sense does any of these methods perform \emph{causal discovery} at all, given that the inputs are
textual correlational premises rather than raw data?  We answer this question
carefully, because the distinction matters for interpreting every row of the table.

\paragraph{What causal discovery requires.}
Causal discovery is the task of recovering a directed acyclic graph $G = (V, E)$
from evidence about the joint distribution $P(V)$.  The foundational difficulty is
that observational evidence alone identifies only a \emph{Markov equivalence class}
(MEC): a set of DAGs that encode identical conditional independencies and are
therefore indistinguishable from passive observation~\citep{scholkopf2021toward}.
Resolving equivalences determining, for instance, whether $A \to C \to B$
or $A \leftarrow C \to B$ generated a given covariance structure requires
reasoning about \emph{interventions} and \emph{counterfactuals}.
This is the precise sense in which the Corr2Cause task is a causal discovery task:
the textual premises encode the full conditional independence structure of the
underlying DAG, and the hypothesis to be evaluated is a structural claim about that
DAG.  The model must go beyond the MEC.

\paragraph{A concrete instance.}
Consider the following sample from the Corr2Cause dataset~\citet{jin2023can}

\begin{quote}
\small
\textit{Premise.}
Suppose there is a closed system of 4 variables $A$, $B$, $C$, and $D$.
All the statistical relations among these 4 variables are as follows:
$A$ correlates with $B$.\ $A$ correlates with $C$.\ $A$ correlates with $D$.\ 
$B$ correlates with $C$.\ $B$ correlates with $D$.\ $C$ correlates with $D$.
However, $B$ and $D$ are independent given $A$.\ 
$B$ and $D$ are independent given $A$ and $C$.\ 
$C$ and $D$ are independent given $A$.\ 
$C$ and $D$ are independent given $A$ and $B$.

\smallskip
\textit{Hypothesis.}
There exists at least one collider (i.e., common effect) of $A$ and $B$.

\smallskip
\textit{Label.} \textbf{0} (False).
\end{quote}

\noindent
The premise encodes a precise pattern of marginal correlations and conditional
independencies.  Both a chain $A \to C \to D$ with $B \leftarrow A$ and a fork
$A \to C$, $A \to D$, $A \to B$ are consistent with these statements and
crucially, \emph{neither} places a collider on the $A$-$B$ path.  A model that
merely memorises surface correlations between hypothesis templates and labels, or
that exploits the co-occurrence statistics of ``correlates with'' and ``collider,''
will answer incorrectly precisely on instances like this one: the collider template
is lexically present (the word appears in the hypothesis), the pairwise correlations are dense (six pairs all correlate), and yet the structural answer is unambiguously negative.  Genuine resolution requires reading the conditional independence pattern as a constraint on the set of consistent DAGs, checking whether the hypothesised v-structure $(\cdot \to B \leftarrow \cdot)$ is compatible with any member of that set, and returning the correct verdict.

\paragraph{The methods in Table~\ref{tab:method_comparison} differ precisely here.}
Zero-shot prompting and ICL present this reasoning problem to the LLM as a
single global prediction.  SFT and DPO train the LLM to associate premise-hypothesis pairs with labels.  As Theorem~\ref{thm:kernel_obstruction} proves, all four approaches produce kernel-type predictors whose score margin on near-miss pairs sharing all but a vanishing fraction of their tokens is bounded by $B\kappa\sqrt{2\delta}$, which collapses to zero as graph complexity grows.  On the instance above, the distinguishing information (the two conditional independencies $B \perp\!\!\!\perp D \mid A$ and $C \perp\!\!\!\perp D \mid A$ that rule out any v-structure on the $A$--$B$ path) occupies fewer than five tokens in a premise of over three hundred.  No bounded-norm kernel predictor can amplify a signal of that relative magnitude to a reliable decision margin.

\noindent
A-CBO escapes this limitation by \emph{never asking the LLM to make this
discrimination}.  Instead, the Bayesian loop poses a sequence of local interventional queries---``Does $D$ change under $\mathrm{do}(A = v)$?'', ``Does $C$ change under $\mathrm{do}(B = v)$?'' whose binary answers are kernel-separated by a constant $\rho > 0$ independent of $\delta$ (Lemma~\ref{lem:interventional_separation}).  Each answer eliminates candidate hypotheses: a ``yes'' to $\mathrm{do}(B)$ affecting $C$ is consistent with a directed path $B \to \cdots \to C$ but inconsistent with $B \leftarrow A \to C$ if the path is blocked by the fork.  After $T^\star = O(\log n / \log\frac{1-\eta}{\eta})$ such queries, the Bayesian posterior concentrates on the surviving DAG class, and the hypothesis is evaluated against it, not by the LLM, but by deterministic graph-mutilation lookup.

\paragraph{The sense in which this is causal discovery.}
A-CBO recovers the Markov equivalence class (Phase~1) and then resolves
equivalences via interventional evidence (Phase~2).  This two-phase structure
mirrors the classical architecture of causal discovery algorithms such as PC and
FCI: constraint-based skeleton recovery followed by orientation via v-structure
detection.  The difference is that A-CBO's ``experiments'' are oracle queries to
a frozen LLM rather than physical interventions on a real system.  This is
legitimate under Assumption~\ref{ass:oracle}: if the LLM reliably reports
interventional effects (with accuracy $1 - \eta > 1/2$), its answers carry the
same logical content as a physical experiment, and the Bayesian update is
mathematically equivalent to Bayesian experimental design over the hypothesis
space.  The discovery is therefore genuine in the formal sense: A-CBO identifies
the correct DAG in a process that depends on interventional reasoning, not on
correlational pattern matching which is precisely what distinguishes causal
discovery from supervised classification.

\section{Extended Related Work}
\label{app:related_work}
\textbf{Causality and LLMs.}
 \citet{scholkopf2021toward} posit that genuine causal reasoning requires modelling interventions and counterfactuals, establishing the theoretical gap between pattern recognition and causal inference. \citet{zecevic2023causal} apply this argument directly to LLMs, introducing meta SCMs to explain why LLMs sometimes appear causal: they recite correlations over causal facts in their training data rather than performing inference, making them ``causal parrots.'' \citet{wu2024causality} move from diagnosis to taxonomy, surveying five distinct stages, from data augmentation to architecture modification, at which causal principles can be integrated into LLMs, and find persistent shortfalls at every stage. \citet{zhang2023understanding} narrow the question to which specific types of causal queries LLMs can and cannot answer, concluding that they succeed at identifying known causal relations from memorized knowledge but fail at inferring novel ones from statistical evidence. \citet{gupta2024essential} argue the problem from the opposite direction: that robust world models must be veridical, faithfully simulating counterfactual consequences rather than pattern-matching, implying that LLMs in their current form cannot serve as causal world models. Concurrently, \citet{wu2025llm} argue on empirical and conceptual grounds that LLMs should be restricted to non-decisional support in causal discovery, meaning they should never determine the existence or directionality of causal relationships. While these works converge on the conclusion that LLMs cannot perform genuine causal reasoning, our contribution differs in kind: we prove a formal impossibility showing that the dominant training paradigms cannot overcome this limitation, and we provide a constructive escape that restricts the LLM to exactly the non-decisional oracle role that \citet{wu2025llm} advocate.
\\\\
\textbf{Benchmarks for causal reasoning in LLMs.} \citet{jin2023can, jin2024corr2cause} introduced Corr2Cause, the first benchmark testing pure causal inference from correlational statements, and showed that seventeen LLMs perform near random while fine-tuned models achieve high in-distribution accuracy but collapse under perturbations such as variable renaming or paraphrasing. CausalProbe-2024~\citep{chi2024unveiling} takes a different approach, using fresh news data unseen during training to distinguish shallow associative reasoning from genuine causal inference, and finds that even state-of-the-art models suffer significant performance drops on novel causal questions, suggesting that strong benchmark performance reflects memorization rather than reasoning. CausalGraph2LLM~\citep{sheth2025causalgraph2llm} shifts focus from the reasoning task to the representation problem, evaluating how different graph encoding strategies affect LLM performance on causal tasks and finding high sensitivity to format choices. \citet{yamin2024failure} examine a complementary failure mechanism, showing that LLMs rely on superficial heuristics such as event ordering rather than structural reasoning when processing causal narratives involving chains, forks, and colliders. These benchmarks document the empirical phenomenon our theory explains: degradation that worsens with structural complexity and near-miss configurations. We extend this empirical foundation with Extended Corr2Cause, scaling from 7 to 24 variables with 18K samples, and show that the degradation pattern matches the quantitative predictions of our kernel obstruction theorem.
\\\\
\textbf{Prompting and Structured Reasoning for Causal Discovery.} Several recent works attempt to improve LLM causal reasoning through better prompting rather than architectural or algorithmic changes. \citet{sgouritsa2024prompting} propose PC-SubQ, which decomposes Corr2Cause into sub-questions aligned with the steps of the PC algorithm, achieving improved and perturbation-robust performance across five LLMs by guiding the model through each algorithmic step sequentially. \citet{sun2025structured} takes a different approach, introducing explicit DAG construction templates that force the LLM to build the causal structure before answering queries, rather than decomposing the query itself. \citet{kadziolka2025causal} combines both ideas, embedding the full PC algorithm within a single unified prompt and evaluating on reasoning-specialist models such as o3-mini and DeepSeek-R1, achieving strong zero-shot baselines without fine-tuning. \citet{li2026leveraging} go further still, integrating LLM outputs into a formal argumentation framework based on d-separation to ensure logical consistency. While these approaches yield meaningful empirical gains, they all share a structural property: the LLM remains the entity that judges whether a causal relationship holds. Our kernel obstruction theorem implies that this design faces a fundamental ceiling on near-miss hypothesis separation, regardless of how the input is decomposed, because the LLM's output is still a kernel-type prediction. A-CBO avoids this ceiling by never asking the LLM to distinguish between competing causal graphs. Instead, it poses simple interventional queries whose answers differ across hypotheses, and delegates the discrimination to an external Bayesian loop outside the kernel predictor's representation space.

\textbf{LLM-Assisted Causal Discovery with Interventions.} LeGIT \citet{li2025can} introduces a framework in which LLMs propose intervention targets based on their world knowledge about variable semantics, effectively warm-starting numerical causal discovery algorithms when limited interventional data makes gradient-based target selection unreliable. MAC \citep{le2024multi} explores multi-agent LLM systems for causal discovery, combining a metadata-driven debate module where multiple LLM agents argue about causal relationships with a statistical execution module that validates claims against data. \citet{abdulaal2024causal} propose a causal agent which iteratively refines causal graphs by alternating between LLM-proposed structures and deep structural causal model fitting. These methods all use LLMs as \emph{reasoning agents} that make or refine causal judgments, meaning the LLM participates in deciding which edges exist. A-CBO differs in both motivation and mechanism: it is derived from our impossibility theorem, which dictates that the discrete causal decision must reside outside the LLM. The LLM serves only as a fixed oracle answering binary interventional queries, while an external Bayesian loop, operating in the probability simplex rather than the model's representation space, performs all hypothesis discrimination. LeGIT is closest in spirit, as it also pairs LLMs with an external loop, but it uses the LLM's domain knowledge to select which interventions to perform, whereas A-CBO uses information-theoretic scoring to select interventions and the LLM only to evaluate their outcomes.
\\\\
\textbf{Neural Tangent Kernels, Lazy Training and In-Context Learning.} Over-parameterised networks in the lazy/NTK regime produce predictions that evolve as kernel smoothers with fixed feature maps~\citep{jacot2018neural}. Recent theoretical work establishes that in-context learning operates within this framework: \citet{han2023explaining} show that Bayesian inference in ICL can be interpreted as kernel regression over demonstrations, \citet{sun2025role} prove that transformer blocks with feed-forward layers implement gradient descent on polynomial kernel regression losses, extending the kernel characterisation from linear to nonlinear settings, and \citet{li2025provable} provide the first formal analysis of in-context learning training dynamics for nonlinear regression, showing convergence to kernel-type solutions. Separately, work on the lazy-rich phase transition~\citep{karkada2024lazy,ghorbani2019limitations} establishes that the kernel regime and feature learning are mutually exclusive: networks must leave the lazy regime to learn new representations. Our kernel obstruction theorem builds directly on the kernel characterisation of SFT and ICL, deriving a negative consequence that these works do not consider: when two causal hypotheses are observationally near-identical, the kernel representation cannot separate them without the RKHS norm growing unboundedly, violating the conditions that define the lazy regime. A natural question is whether leaving the lazy regime for the rich/feature-learning regime could resolve the obstruction; A-CBO offers a cheaper escape that preserves the lazy regime entirely.

\textbf{Causal Bayesian optimisation and active experimental design.}
\citet{scherrer2021learning} introduced active intervention targeting for maximising information gain about causal structure, selecting interventions that most reduce entropy over the hypothesis space. \citet{agrawal2019abcd} proposed the ABCD-Strategy for budgeted experimental design, providing submodularity-based guarantees on the number of experiments needed to identify the causal graph. Both methods assume access to a physical system where interventions produce real observational outcomes, meaning an experimenter can actually set a variable to a value and measure the downstream effects. A-CBO adapts their information-theoretic framework to a fundamentally different setting where a frozen LLM replaces the physical experiment, answering interventional queries through forward passes rather than data collection. This substitution is justified by our analysis showing that interventional signals remain bounded away from zero even when observational signals collapse, making binary oracle responses sufficient for hypothesis discrimination despite the oracle's imperfection.

\section{Full Proofs}
\label{app:proofs}

\subsection{Proof of Lemma~\ref{lem:token_overlap} (Near-Miss Kernel Similarity)}

\begin{lemma}[\textbf{Token overlap implies kernel similarity}]
Let $\chi^+ = (x, y^+)$ and $\chi^- = (x, y^-)$ share a token prefix of length $\ell$ out of
total sequence length $L$, with the remaining $L - \ell$ tokens differing. Under the NTK kernel 
induced by a transformer with bounded Jacobian norms and causal attention, the normalised kernel 
similarity satisfies:
\begin{equation}
    \frac{K(\chi^+, \chi^-)}{\sqrt{K(\chi^+,\chi^+) \cdot K(\chi^-, \chi^-)}} 
    \geq 1 - C \cdot \frac{L - \ell}{L}
\end{equation}
for a constant $C > 0$ depending only on $\kappa$. Hence $\delta \leq C(L-\ell)/L$.
\end{lemma}

\begin{proof}
Decompose the Jacobian inner product along the token sequence. At each position $t$, the 
contribution to $K(\chi^+, \chi^-)$ is:
\[
\langle \nabla_\theta z_{\theta,t}(\chi^+),\, \nabla_\theta z_{\theta,t}(\chi^-) \rangle.
\]
Under causal attention, position $t$ attends only to positions $\leq t$. Consequently:

\begin{itemize}
    \item At shared positions $t \leq \ell$: inputs are identical so the Jacobians coincide, 
    and each contribution equals $\|\nabla_\theta z_{\theta,t}\|^2 \geq 0$.
    
    \item At differing positions $t > \ell$: by Cauchy--Schwarz and the diagonal bound 
    $K(\chi, \chi) \leq \kappa^2$, each contribution is bounded below by $-\kappa^2 / L$.
\end{itemize}

Summing over all $L$ positions:
\begin{equation}
    K(\chi^+, \chi^-) \geq K(\chi^+, \chi^+) - \frac{2\kappa^2(L - \ell)}{L}.
\end{equation}
Dividing both sides by $\sqrt{K(\chi^+,\chi^+) \cdot K(\chi^-,\chi^-)}$ and applying 
$K(\chi^\pm, \chi^\pm) \leq \kappa^2$ gives:
\[
\frac{K(\chi^+, \chi^-)}{\sqrt{K(\chi^+,\chi^+) \cdot K(\chi^-, \chi^-)}} 
\geq 1 - \frac{2\kappa^2(L-\ell)}{L \cdot K(\chi^+,\chi^+)}
\geq 1 - \frac{2(L-\ell)}{L},
\]
yielding the stated bound with $C = 2$.
\end{proof}

\begin{remark}
The positional decomposition assumes approximately uniform Jacobian norm across positions, which 
is standard in NTK analyses of transformers~\citep{jacot2018neural}. For near-miss causal instances, $L = O(d^2)$ (the premise contains $O(d^2)$ correlation statements) and 
$L - \ell = O(1)$ (the reasoning chains diverge only in the final causal inference step), giving 
$\delta \leq O(1/d^2)$ as claimed in the main text.
\end{remark}

\subsection{Proof of Theorem~\ref{thm:kernel_obstruction} (Kernel Obstruction)}

\begin{proof}[Proof of Theorem~\ref{thm:kernel_obstruction}]
By Cauchy--Schwarz,
$|s(\chi^+)-s(\chi^-)| = |\langle w,\varphi(\chi^+)-\varphi(\chi^-)\rangle| \le \|w\|_\RKHS\cdot\|\varphi(\chi^+)-\varphi(\chi^-)\|_\RKHS$.

Expanding:
\begin{align}
\|\varphi(\chi^+)-\varphi(\chi^-)\|^2 &= K(\chi^+,\chi^+) + K(\chi^-,\chi^-) - 2K(\chi^+,\chi^-).
\end{align}
Let $a=\sqrt{K(\chi^+,\chi^+)},\, b=\sqrt{K(\chi^-,\chi^-)}$ with $a,b\le\kappa$.
The $\delta$-similarity gives $K(\chi^+,\chi^-)\ge(1-\delta)ab$, so:
\begin{align}
\|\varphi(\chi^+)-\varphi(\chi^-)\|^2 &\le a^2+b^2-2(1-\delta)ab = (a-b)^2 + 2\delta ab \le 2\kappa^2\delta.
\end{align}
Therefore $|s(\chi^+)-s(\chi^-)| \le B\kappa\sqrt{2\delta}$.
Requiring $\gamma \le B\kappa\sqrt{2\delta}$ gives $B \ge \gamma/(\kappa\sqrt{2\delta})$.
\end{proof}

\subsection{Proof of Lemma~\ref{lem:interventional_separation} (Interventional Kernel Separation)}

\begin{proof}[Proof of Lemma~\ref{lem:interventional_separation}]
Apply Theorem~\ref{thm:kernel_obstruction} to the interventionally augmented inputs.
The bound becomes $B\kappa\sqrt{2\rho}$ where $\rho$ is the interventional sensitivity parameter.
The key content is the \emph{decoupling}: $\rho$ depends on the structural difference between hypotheses (via the discrimination set $\mathcal{D}(G^+,G^-)$), while $\delta$ depends on observational similarity.
Since $(V_i,V_j)\in\mathcal{D}(G^+,G^-)$ is chosen adversarially against the hypotheses, $\rho\not\to 0$ as $\delta\to 0$.
\end{proof}

\subsection{Proof of Theorem~\ref{thm:convergence} (Convergence of A-CBO)}

\begin{proof}[Proof of Theorem~\ref{thm:convergence}]
Fix a wrong hypothesis $G_k \neq G^\star$. At each round $t$, the oracle answers correctly with probability $1-\eta > 1/2$. If the oracle answers correctly:
\begin{itemize}
    \item If $\hat{r}_k \neq r^\text{obs}$: $\pi_k^{(t)} \propto \pi_k^{(t-1)} \cdot \eta$, so $\pi_k$ is multiplied by $\eta < 1/2$.
    \item If $\hat{r}_{G^\star} = r^\text{obs}$: $\pi_{G^\star}^{(t)} \propto \pi_{G^\star}^{(t-1)} \cdot (1-\eta)$, so $\pi_{G^\star}$ is multiplied by $(1-\eta) > 1/2$.
\end{itemize}
The ratio $\pi_{G^\star}^{(t)} / \pi_{G_k}^{(t)}$ grows by a factor of $(1-\eta)/\eta > 1$ in expectation at each round where the intervention distinguishes $G^\star$ from $G_k$. Since $\mathcal{D}(G^\star, G_k) \neq \emptyset$, such an intervention is always available (information gain is positive). After $T^\star = \lceil \log n / \log((1-\eta)/\eta) \rceil$ rounds, for each $k \neq \star$:
\[
\frac{\pi_{G^\star}^{(T^\star)}}{\pi_{G_k}^{(T^\star)}} \geq \left(\frac{1-\eta}{\eta}\right)^{T^\star} \geq n,
\]
with probability at least $1 - \eta^{T^\star}$ by a Chernoff-type bound on the Bernoulli oracle. A union bound over all $n-1$ wrong hypotheses gives success probability $\geq 1 - n\eta^{T^\star}$. The algorithm operates only in $\Delta^{n-1}$ and never modifies $\theta_0$, so the lazy/NTK regime is preserved exactly (Proposition~\ref{prop:lazy_preserving}). The bound $T^\star$ depends only on $n$ and $\eta$, not on the near-miss parameter $\delta$.
\end{proof}

\subsection{SFT Kernel Derivation}
\label{app:sft_kernel}

Given input $x_u$ and target $y_u = (y_{u,1}, \ldots, y_{u,L})$, SFT minimises:
\begin{equation}
    \mathcal{L}_{\text{SFT}}(\theta;\, x_u, y_u) 
    = -\sum_{l=1}^{L} \log \pi_\theta(y_{u,l} \mid x_u, y_{u,<l}).
\end{equation}
The parameter gradient decomposes as:
\begin{equation}
    \nabla_\theta \mathcal{L}_{\text{SFT}} 
    = \sum_{l=1}^{L} (\nabla_\theta z_{\theta,l})^\top 
      \bigl(\pi_\theta(\cdot \mid x_u, y_{u,<l}) - e_{y_{u,l}}\bigr),
\end{equation}
where $e_{y_{u,l}}$ is the one-hot vector at token $y_{u,l}$. A first-order Taylor expansion of 
the effect on a test sequence $(x_o, y_o)$ yields:
\begin{equation}
\label{eq:sft_ntk}
    \Delta \log \pi(y_o \mid x_o) 
    \approx -\eta \sum_{l,l'} 
    \langle \nabla_\theta z_{\theta,l}(x_o, y_{o,<l}),\, 
            \nabla_\theta z_{\theta,l'}(x_u, y_{u,<l'}) \rangle 
    \cdot \bigl(\pi_\theta(\cdot \mid x_u, y_{u,<l'}) - e_{y_{u,l'}}\bigr),
\end{equation}
where the inner product of Jacobians is the empirical NTK between the test and training sequences. 
The resulting predictor is therefore a kernel-type scorer with $\|w\|_{\mathcal{H}}$ controlled 
by the learning rate $\eta$ and training duration, confirming $B = O(1)$ under the lazy regime 
assumption.

\section{Lazy-Preserving Property of A-CBO}
\begin{proposition}[Lazy-Preserving Property of A-CBO]
\label{prop:lazy_preserving}
A-CBO satisfies three properties simultaneously:
\begin{enumerate}
    \item \textbf{No gradient updates:} the LLM parameters $\theta_0$ are frozen throughout.
    \item \textbf{Kernel-regime queries:} each LLM call is a standard forward pass; the NTK 
    characterisation applies.
    \item \textbf{Extra-kernel decision:} the causal judgment is made by the Bayesian posterior 
    update $\pi^{(t+1)} \propto \pi^{(t)} \odot \ell^{(t)}$ in $\Delta^{n-1}$, a space not 
    subject to Theorem~\ref{thm:kernel_obstruction}.
\end{enumerate}
\end{proposition}
\vspace{-0.6cm}
\begin{proof}
Items (1) and (2) follow from inspection of Algorithm~\ref{alg:acbo}: the only interactions with 
$\mathcal{L}$ are forward-pass queries. Item (3) holds because the belief update is a deterministic 
operation on a probability vector performed by the agentic controller. The simplex $\Delta^{n-1}$ 
admits arbitrary discrete concentration without any object in $\mathcal{H}$ changing.
\end{proof}
A-CBO avoids kernel obstruction entirely by encoding the discrete decision in $\pi \in \Delta^{n-1}$ instead of $w \in \mathcal{H}$. The  LLM is used \emph{only for what it can do within the kernel regime} i.e. answering local interventional yes/no questions, not for making global near-miss causal judgments. We summarize this in Appendix~\ref{app:dichotomy}.

\section{Dichotomy Summary}
\label{app:dichotomy}

\begin{center}
\small
\begin{tabular}{@{}p{0.44\linewidth} p{0.50\linewidth}@{}}
\toprule
\textbf{Theorem~\ref{thm:kernel_obstruction} (negative)} & \textbf{Lemma~\ref{lem:interventional_separation} + Prop.~\ref{prop:lazy_preserving} (positive)} \\
\midrule
Within a bounded-norm kernel predictor, near-miss pairs cannot be separated by more than $B\kappa\sqrt{2\delta}$. & Around a frozen kernel predictor, an agentic loop separates any interventionally distinguishable hypotheses with strength controlled by $\rho$ and $\eta$, independently of $\delta$. \\
\addlinespace
Applies to SFT (lazy), DPO, ICL. & Applies to A-CBO with any frozen LLM. \\
\addlinespace
Discrete decision encoded in $w\in\RKHS$. & Discrete decision encoded in $\bm{\pi}\in\Delta^{n-1}$. \\
\addlinespace
Escape requires leaving the lazy regime. & Lazy regime preserved exactly. \\
\bottomrule
\end{tabular}
\end{center}

\section{Detailed Experimental Setup}
\label{app:detail_exp}
Our Experimental setup: benchmarks, baselines, A-CBO models and
hyperparameters, fine-tuning protocol, and evaluation metrics are mentioned in this Table~\ref{tab:setup}
% ── TABLE 1: Experimental Setup ───────────────────────────
\begin{table}[htbp]
\centering
\caption{Experimental setup: benchmarks, baselines, A-CBO models and
hyperparameters, fine-tuning protocol, and evaluation metrics.}
\label{tab:setup}
\vspace{2pt}
\small
\setlength{\tabcolsep}{5pt}
\renewcommand{\arraystretch}{1.18}
\begin{tabular}{@{} l l p{7.8cm} @{}}
\toprule
\textbf{Category} & \textbf{Entry} & \textbf{Details} \\
\midrule
\textbf{Datasets}
  & \textsc{Corr2Cause}        & 7,524 test samples, $d\!=\!2$--$6$; six causal relation templates (\texttt{Parent}, \texttt{Child}, \texttt{Ancestor}, \texttt{Descendant}, \texttt{Collider}, \texttt{Confounder}); macro-avg.\ F1~\citep{jin2024corr2cause}. \\
  & \textsc{Ext.\ C2C} (ours)  & 18,000 samples, $d\!=\!7$--$24$ (1K/depth); all-negative labels; binary rejection accuracy; near-miss gap $\propto O(1/d^2)$. \\
\midrule
\textbf{Baselines}
  & Zero-shot GPT-4            & Direct prompting, no A-CBO loop; F1\,=\,29.1. \\
  & LLaMA-7B (FT)              & Finetuned on 197,634 in-dist.\ samples; F1\,=\,92.0. \\
  & RoBERTa-L SFT              & 355M params; 1.3M Ext.\ C2C samples, 3 epochs, AdamW, lr\,=\,$2{\times}10^{-5}$, batch 32. \\
  & RoBERTa-L DPO              & Preference pairs per instance, $\beta\!=\!0.1$; $4{\times}$A100. \\
\midrule
\textbf{A-CBO Models}
  & \textit{High}              & GLM-5.1$^\star$,\ Qwen3-30B$^\star$ \quad (thinking mode enabled) \\
  & \textit{Mid}               & Qwen3.5-122B,\ Llama-3.3-70B \\
  & \textit{Low}               & Gemma-3-12B-IT,\ LLaMA-7B \\
\midrule
\textbf{Hyperparameters}
  & Budget / Noise             & $T\!=\!20$,\ $\varepsilon\!=\!0.1$ (random fraction),\ $\eta\!=\!0.1$ (oracle noise) \\
  & Stopping / Voting          & $\delta_c\!=\!0.01$ (entropy threshold),\ $M\!=\!3$ (majority votes),\ $N\!=\!8$ (candidates) \\
\midrule
\textbf{Metrics}
  & Orig.\ / Extended          & Macro-avg.\ F1 over six relation classes\,/\,binary rejection accuracy \\
\bottomrule
\end{tabular}
\end{table}

\section{Dataset Statistics}
\label{app:dataset}

\begin{table*}[htbp]
\centering
\caption{Statistics of the \textsc{Extended Corr2Cause} dataset by graph size $d$ (number of variables).
The dataset extends the original Corr2Cause ($d\!=\!2{-}6$, $\sim$200K samples) to larger causal graphs ($d\!=\!7{-}28$), totaling over \textbf{2.17M} samples.
Each graph size contains an equal distribution across all six causal relation templates.}
\label{tab:dataset_stats}
\vspace{4pt}
\renewcommand{\arraystretch}{1.12}
\centering
\resizebox{\textwidth}{!}{%
\begin{tabular}{@{} c r r r r r r r @{}}
\toprule
$\boldsymbol{d}$
  & \textbf{\# Samples}
  & \textbf{\# Train}
  & \textbf{\# Dev}
  & \textbf{\# Test}
  & \textbf{\# Tok\,/\,Premise}
  & \textbf{\% Positive}
  & \textbf{Vocab} \\
\midrule

\rowcolor{blue!5}
7  & 12,600  & 10,600  & 1,000 & 1,000 & 112.0  & 16.25 & 65 \\
\rowcolor{blue!5}
8  & 16,800  & 14,800  & 1,000 & 1,000 & 140.8  & 19.78 & 67 \\
\rowcolor{blue!5}
9  & 21,600  & 19,600  & 1,000 & 1,000 & 173.7  & 23.48 & 69 \\
\rowcolor{blue!5}
10 & 27,000  & 25,000  & 1,000 & 1,000 & 210.9  & 26.66 & 71 \\

\rowcolor{orange!5}
11 & 33,000  & 31,000  & 1,000 & 1,000 & 251.9  & 29.70 & 73 \\
\rowcolor{orange!5}
12 & 39,600  & 37,600  & 1,000 & 1,000 & 297.0  & 32.60 & 75 \\
\rowcolor{orange!5}
13 & 46,800  & 44,800  & 1,000 & 1,000 & 346.0  & 35.36 & 77 \\
\rowcolor{orange!5}
14 & 54,600  & 52,600  & 1,000 & 1,000 & 398.8  & 37.98 & 79 \\
\rowcolor{orange!5}
15 & 63,000  & 61,000  & 1,000 & 1,000 & 456.1  & 40.46 & 81 \\

\rowcolor{red!4}
16 & 72,000  & 70,000  & 1,000 & 1,000 & 516.9  & 42.80 & 83 \\
\rowcolor{red!4}
17 & 81,600  & 79,600  & 1,000 & 1,000 & 582.1  & 45.00 & 85 \\
\rowcolor{red!4}
18 & 91,800  & 89,800  & 1,000 & 1,000 & 650.8  & 46.00 & 87 \\
\rowcolor{red!4}
19 & 102,600 & 100,600 & 1,000 & 1,000 & 723.8  & 46.00 & 89 \\
\rowcolor{red!4}
20 & 114,000 & 112,000 & 1,000 & 1,000 & 800.9  & 46.00 & 91 \\

\rowcolor{violet!5}
21 & 126,000 & 124,000 & 1,000 & 1,000 & 882.1  & 46.00 & 93 \\
\rowcolor{violet!5}
22 & 138,600 & 136,600 & 1,000 & 1,000 & 966.9  & 46.00 & 95 \\
\rowcolor{violet!5}
23 & 151,800 & 149,800 & 1,000 & 1,000 & 1,055.9 & 46.00 & 97 \\
\rowcolor{violet!5}
24 & 165,600 & 163,600 & 1,000 & 1,000 & 1,148.9 & 46.00 & 99 \\

\rowcolor{gray!8}
25 & 180,000 & 180,000$^\dagger$ & --- & --- & 1,246.0 & 46.00 & 101 \\
\rowcolor{gray!8}
26 & 195,000 & 195,000$^\dagger$ & --- & --- & 1,347.2 & 46.00 & 103 \\
\rowcolor{gray!8}
27 & 210,600 & 210,600$^\dagger$ & --- & --- & 1,452.5 & 46.00 & 105 \\
\rowcolor{gray!8}
28 & 226,800 & 226,800$^\dagger$ & --- & --- & 1,561.9 & 46.00 & 107 \\

\midrule
\textbf{All} & \textbf{2,171,400} & \textbf{2,135,400} & \textbf{18,000} & \textbf{18,000}
  & \textbf{991.5} & \textbf{38.64} & \textbf{107} \\
\bottomrule
\end{tabular}%
}

\vspace{4pt}
{\footnotesize
  \# Tok\,/\,Hypothesis is constant at 8.7 across all $d$ (omitted for brevity).
  All six causal templates (parent, child, non-parent ancestor, non-child descendant, has\_collider, has\_confounder) are uniformly distributed ($\tfrac{1}{6}$ each).\\[1pt]
  $\dagger$\,$d=25{-}28$ provided as raw JSON with full graph metadata; no pre-split Dev/Test partition.\\[2pt]
  \colorbox{blue!5}{\strut}\,Small ($d\!=\!7{-}10$)\quad
  \colorbox{orange!5}{\strut}\,Medium ($d\!=\!11{-}15$)\quad
  \colorbox{red!4}{\strut}\,Large ($d\!=\!16{-}20$)\quad
  \colorbox{violet!5}{\strut}\,XL ($d\!=\!21{-}24$)\quad
  \colorbox{gray!8}{\strut}\,Extended ($d\!=\!25{-}28$)
}
\end{table*}

\section{Additional Figures}
\label{app:additional_figures}

\begin{figure}[h]
\centering
\includegraphics[width=0.55\textwidth]{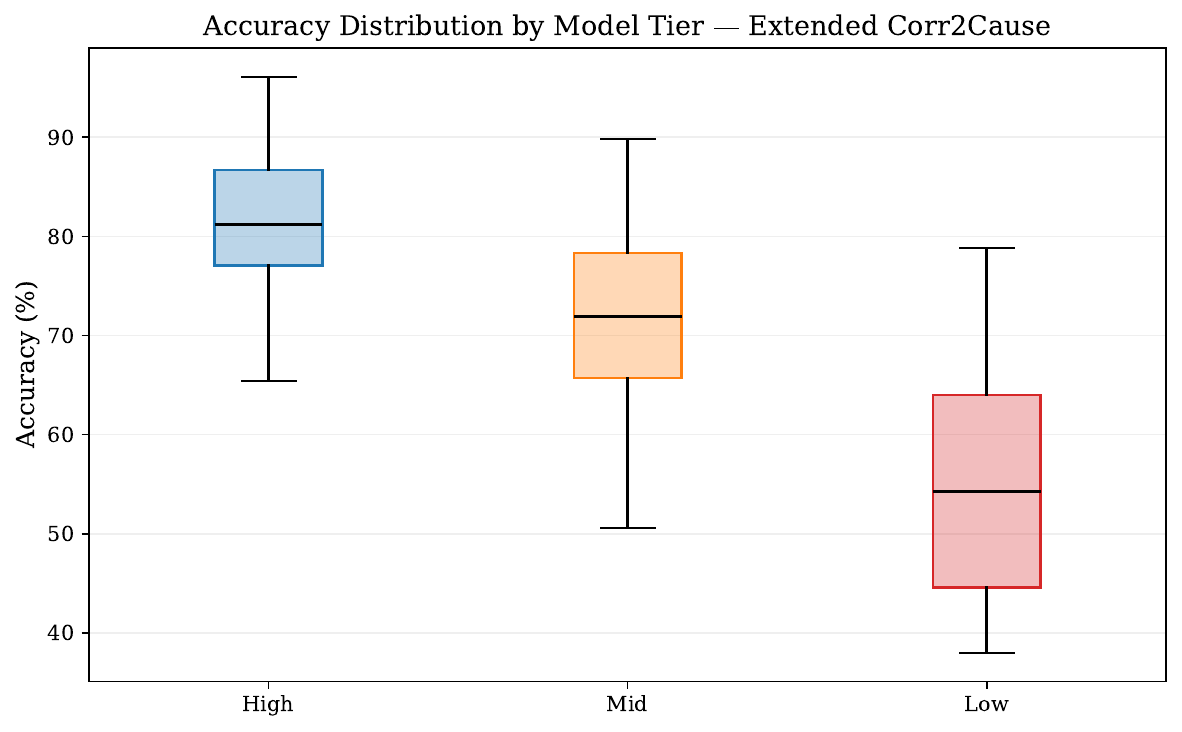}
\caption{Accuracy distribution by model tier on \textsc{Extended Corr2Cause}. Tiers are well-separated with increasing variance at lower tiers (IQR: 8.2 vs.\ 12.1 vs.\ 18.6).}
\label{fig:boxplot}
\end{figure}

\end{document}